\documentclass[twoside]{article}

\usepackage{aistats2024}
\usepackage{algorithm}
\usepackage{algorithmic}
\usepackage{thmtools}
\usepackage{thm-restate}
\usepackage{bbm}
\usepackage{amsmath,amssymb,amsfonts,amsthm}
\newcommand{\argmax}{\mathop{\rm arg~max}\limits}
\newcommand{\argmin}{\mathop{\rm arg~min}\limits}

\newtheorem{theorem}{Theorem}
\newtheorem{lemma}{Lemma}
\newtheorem{definition}{Definition}
\usepackage{subcaption}
\usepackage{graphicx}

\usepackage[round]{natbib}

\begin{document}



\twocolumn[

\aistatstitle{Fixed-Budget Real-Valued Combinatorial Pure Exploration of Multi-Armed Bandit}

\aistatsauthor{ Shintaro Nakamura \And Masashi Sugiyama }

\aistatsaddress{ The University of Tokyo \\ RIKEN AIP \And  RIKEN AIP \\ The University of Tokyo } ]

\begin{abstract}
  We study the real-valued combinatorial pure exploration of the multi-armed bandit in the fixed-budget setting. We first introduce the Combinatorial Successive Asign (CSA) algorithm, which is the first algorithm that can identify the best action even when the size of the action class is exponentially large with respect to the number of arms. We show that the upper bound of the probability of error of the CSA algorithm matches a lower bound up to a logarithmic factor in the exponent. Then, we introduce another algorithm named the Minimax Combinatorial Successive Accepts and Rejects (Minimax-CombSAR) algorithm for the case where the size of the action class is polynomial, and show that it is optimal, which matches a lower bound. Finally, we experimentally compare the algorithms with previous methods and show that our algorithm performs better.
\end{abstract}

\section{Introduction}
The multi-armed bandit (MAB)  model is an important framework in online learning since it is useful to investigate the trade-off between exploration and exploitation in decision-making problems \citep{Auer2002,Audibert2009}. Although investigating this trade-off is intrinsic in many applications, some application domains only focus on obtaining the optimal object, e.g., an arm or a set of arms, among a set of candidates, and do not care about the loss or rewards that occur during the exploration procedure. This learning problem called the pure exploration (PE) task has received much attention \citep{Bubeck2009,Audibert2010}. \par
One of the important sub-fields among PE of MAB is the \emph{combinatorial pure exploration} of the MAB (CPE-MAB)\citep{SChen2014,Gabillon16,LChen2017}. 
In the CPE-MAB, we have a set of $d$ stochastic arms, where the reward of each arm $s\in\{ 1, \ldots, d\}$ follows an unknown distribution with mean $\mu_s$, and an \emph{action class} $\mathcal{A}$, which is a collection of subsets of arms with certain combinatorial structures. 
Then, the goal is to identify the best action from the action class $\mathcal{A}$ by pulling a single arm each round. 
There are mainly two settings in the CPE-MAB. One is the \emph{fixed confidence} setting, where the player tries to identify the optimal action with high probability with as few rounds as possible, and the other is the \emph{fixed-budget} setting, where the player tries to identify the optimal action with a fixed number of rounds \citep{SChen2014,Katz-Samuels2020,WangAndZhu}.
Abstractly, the goal is to identify $\boldsymbol{\pi}^{*}$, which is the optimal solution for the following constraint optimization problem:
\begin{equation}
\begin{array}{ll@{}ll}
\text{maximize}_{\boldsymbol{\pi}}  & \boldsymbol{\mu}^{\top}\boldsymbol{\pi}&\\
\text{subject to}& \boldsymbol{\pi}\in \mathcal{A},   & 
\end{array}\label{AbstractFormulation}
\end{equation}
    where $\boldsymbol{\mu}$ is a vector whose $s$-th element is the mean reward of arm $s$ and $\top$ denotes the transpose. \par
% \begin{figure}[t]
%   \begin{minipage}[b]{0.45\linewidth}
%     \centering
%     \includegraphics[width = \linewidth]{figures/introduction/graph.pdf}
%     \caption{A simple sketch of the shortest path problem.}
%     \label{ShortestPathFigure}
%   \end{minipage}
%   \hspace{0.04\linewidth}
%   \begin{minipage}[b]{0.45\linewidth}
%     \centering
%     \includegraphics[width = \linewidth]{figures/introduction/knapsack.pdf}
%     \caption{A simple sketch of the knapsack problem. We want to know how many of each item should be included in the bag to maximize the total value. Here, the total weight of every item cannot exceed 10kg, which is the capacity of the bag.}
%     \label{KnapsackFigure}
%   \end{minipage}
% \end{figure}
\begin{figure}[t]
    \centering
    \includegraphics[width = 0.5 \linewidth]{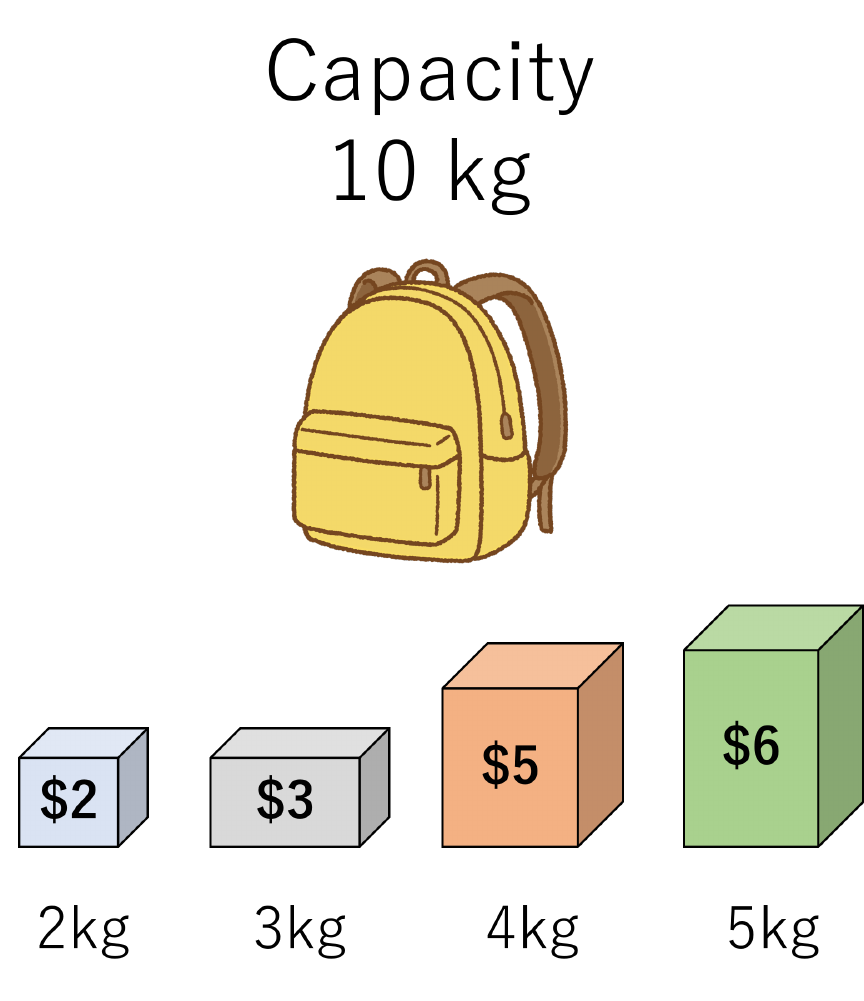}
    \caption{A simple sketch of the knapsack problem. We want to know how many of each item should be included in the bag to maximize the total value. Here, the total weight of every item cannot exceed 10kg, which is the capacity of the bag.}
    \label{KnapsackFigure}
\end{figure}
% One example CPE-MAB can be applied to is the shortest path problem \citep{Sniedovich2006} (Figure \ref{ShortestPathFigure}). In real-world applications, each edge (road) can be congested, and the cost can stochastically change. Therefore, one may observe the cost of roads and want to identify the shortest path with the least expected cost. In this example, each edge $i\in\{1, \ldots, 7\}$ can be seen as an arm, and $\mathcal{A} = \{(1, 0, 1, 0, 0, 1, 0), (0, 1, 0, 1, 0, 1, 0), (0, 1, 0, 0, 1, 0, 1)\}$ in this example.\par
Although CPE-MAB can be applied to many models which can be formulated as (\ref{AbstractFormulation}), most of the existing works in CPE-MAB \citep{SChen2014,WangAndZhu2022,Gabillon16,LChen2017,YihanDu2021,LChen2016} assume $\mathcal{A}\subseteq\{0, 1\}^d$. This means that although we can apply the CPE-MAB framework to the shortest path problem \citep{Sniedovich2006}, top-$K$ arms identification \citep{Kalyanakrishnan2010}, matching \citep{Gibbons1985}, and spanning trees \citep{Pettie2002}, we cannot apply it to problems where $\mathcal{A} \subset \mathbb{R}^d$, such as the optimal transport problem \citep{villani2008}, the knapsack problem \citep{Dantzig2007}, and the production planning problem \citep{Pochet2010}. For instance, in the knapsack problem shown in Figure \ref{KnapsackFigure}, actions are not binary vectors since, for each item, we can put more than one in the bag, e.g., one blue one and two orange ones. \par
To overcome this limitation, \citet{Nakamura2023} introduced the real-valued CPE-MAB (R-CPE-MAB), where the action class is a set of real vectors, i.e., $\mathcal{A}\subset\mathbb{R}^d$. Though they have investigated the R-CPE-MAB in the fixed confidence setting, there is still room for investigation in the fixed-budget setting. To the best of our knowledge, the only existing work that can be applied to the fixed-budget R-CPE-MAB is the Peace algorithm introduced by \citet{Katz-Samuels2020}. However, there are some issues with the Peace algorithm. Firstly, since the Peace algorithm requires enumerating all the actions in advance, it is impractical when $\mathcal{A}$ is exponentially large in $d$. 
Secondly, there is a hyper-parameter for it, and one needs to carefully choose them for the feasibility of the Peace algorithm \citep{Yang2022}. 
Finally, it needs an assumption that the rewards follow a standard normal distribution, which may not be satisfied in real-world applications. \par
In this work, we first introduce a parameter-free algorithm named the Combinatorial Successive Asign (CSA) algorithm, which is a generalized version of the Combinatorial Successive Accepts and Rejects (CSAR) algorithm proposed by \citeauthor{SChen2014}(\citeyear{SChen2014}) for the ordinary CPE-MAB. The CSA algorithm is the first algorithm that can be applied to the fixed-budget R-CPE-MAB even when the size of the action class $\mathcal{A}$ is exponentially large in $d$. We show that the upper bound of the probability of error of the CSA algorithm matches a lower bound up to a logarithmic factor in the exponent. \par
% To let the CSA algorithm work, we need two assumptions. One is the existence of the \emph{offline oracle}, which outputs $\boldsymbol{\pi}^{*}({\boldsymbol{\nu}})~=~\argmax_{\boldsymbol{\pi} \in \mathcal{A}} \boldsymbol{\nu}^{\top}\boldsymbol{\pi}$ in polynomial or pseudo-polynomial time in $d$ once $\boldsymbol{\nu}$ is given. This assumption can often be seen in the CPE-MAB.
% For instance, in the CPE-MAB, we have Dijkstra's algorithm for the shortest path problem \citep{Dijkstra1959}, the Hungarian algorithm for the maximum-weight matching \citep{Kuhn2010}, and Kruskal's algorithm for the minimum spanning tree problem \citep{Kruskal1956}, which are all polynomial-time algorithms with respect to the dimension~$d$. In Section \ref{CSA_Algorithm_Section}, we show that we can assume the existence of offline oracles in some problems in the R-CPE-MAB as well, such as in the knapsack problem and optimal transport problem. The other assumption is that, for any $s \in \{1, \ldots, d\}$, we have \texttt{POSSIBLE-PI($s$)}, which is a set of all possible values that an action in the action set can take as the $s$-th element. In Section \ref{CSA_Algorithm_Section}, we show that this assumption also holds for some problem instances, such as in the knapsack problem and optimal transport problem. \par
Since the CSA algorithm does not match a lower bound, we introduce another algorithm named the Minimax Combinatorial Successive Accepts and Rejects (Minimax-CombSAR) algorithm inspired by \citet{Yang2022} for the case where the size of the action class is polynomial. In Section \ref{CSA_Algorithm_Section}, we show that the Minimax-CombSAR algorithm is optimal, which means that the upper bound of the probability of error of the best action matches a lower bound. We also show that the Minimax-CombSAR algorithm has only one hyper-parameter, and is easily interpreted. \par
% Also, the Minimax-CombSAR algorithm has a milder assumption on the reward distribution than the Peace algorithm. It only assumes that reward distributions have $R$-sub-Gaussian tails for some known constant $R>0$. This assumption lets us apply the Minmimax-CombSAR algorithm to a broader class than the Peace algorithm in R-CPE-MAB since distributions with $R$-sub-Gaussian tails include the Gaussian distributions with variance smaller or equal to $R^2$ and any bounded distributions on intervals whose length is no larger than $\frac{R}{2}$. Finally, we show that the Minimax-CombSAR algorithm has a tighter upper bound of the probability of error of the best action than the Peace algorithm.  \par

% We also discuss the lower bound of the probability of error of the best action for both cases when the action class $\mathcal{A}$ is polynomially large in $d$ and exponentially large in $d$. We introduce some notions that we expect to play an essential role in characterizing the difficulty of the problem instance in the R-CPE-MAB, which were not discussed in \citet{Nakamura2023}. \par
Finally, we report the results of numerical experiments. First, we show that the CSA algorithm can identify the best action in a knapsack problem, where the size of the action class can be exponentially large in $d$. Then, when the size of the action class is polynomial in $d$, we show that the Minimax-CombSAR algorithm performs better than the CSA algorithm and the Peace algorithm in identifying the best action in a knapsack problem. \par
% Below is the structure of this paper. In Section \ref{ProblemFormulation}, we formally introduce our model, R-CPE-MAB. In Section \ref{Lower_Bound_Section}, we introduce some key notions and show a lower bound of the R-CPE-MAB. In Section \ref{CSA_Algorithm_Section} and Section \ref{Minimax-CombSAR_Algorithm_Section}, we introduce the CSA algorithm and the Minimax-CombSAR algorithm, respectively, and analyze their probability of errors of the best action.
% Finally, in Section \ref{ExperimentSection}, we compare the algorithms numerically and conclude this study in Section~\ref{Conclusion_Section}.

\section{Problem Formulation} \label{ProblemFormulation}
In this section, we formally define our R-CPE-MAB model. Suppose we have $d$ arms, numbered $1, \ldots, d$. Assume that each arm $s\in[d]$ is associated with a reward distribution $\phi_s$, where $[d] = \{1, \ldots, d\}$. We assume all reward distributions have $R$-sub-Gaussian tails for some known constant $R > 0$. Formally, if $X$ is a random variable drawn from $\phi_s$ for some $s\in[d]$, then, for all $\lambda\in \mathbb{R}$, we have $\mathbb{E}[\exp (\lambda X - \lambda\mathbb{E}[X])] \leq \exp (R^2\lambda^2/2)$. It is known that the family of $R$-sub-Gaussian tail distributions includes all distributions that are supported on $[a, b]$, where $(b - a)/2 = R$, and also many unbounded distributions such as Gaussian distributions with variance $R^2$ \citep{Rivasplata2012,Rinaldo2018}.
Let $\boldsymbol{\mu} = (\mu_1, \ldots, \mu_d)^{\top}$ denote the vector of expected rewards, where each element $\mu_s = \mathbb{E}_{X\sim\phi_s}[X]$ denotes the expected reward of arm $s$ and $\top$ denotes the transpose. 
% We denote the number of times arm $s$ pulled before round $t$ by $T_s(t)$, and by $\boldsymbol{\hat{\mu}}(t) = (\hat{\mu}_1(t), \ldots, \hat{\mu}_d(t))^{\top}$ the vector of sample means of each arm in round $t$.\par
With a given $\boldsymbol{\nu}$, let us consider the following linear optimization problem:
\begin{equation}\label{Abstract_Formulation_IntegerProgramminProblem}
\begin{array}{ll@{}ll}
\text{maximize}_{\boldsymbol{\pi}}  & \boldsymbol{\nu}^{\top}\boldsymbol{\pi}&\\
\text{subject to}& \boldsymbol{\pi}\in \mathcal{C}\subset\mathbb{R}_{\geq 0}^d.   &
\end{array}
\end{equation}
Here, $\mathcal{C}$ is a problem-dependent feasible region of $\boldsymbol{\pi}$, which satisfies some combinatorial structures. Then, for any $\boldsymbol{\nu}\in \mathbb{R}^d$, we denote by $\boldsymbol{\pi}^{\boldsymbol{\nu}, \mathcal{C}}$ the solution of (\ref{Abstract_Formulation_IntegerProgramminProblem}). 
We define the action class $\mathcal{A}$ as the set of vectors that contains optimal solutions of (\ref{Abstract_Formulation_IntegerProgramminProblem}) for any $\boldsymbol{\nu}$, i.e., 
\begin{equation}
    \mathcal{A} = \left\{ \boldsymbol{\pi}^{\boldsymbol{\nu, \mathcal{C}}} \in \mathbb{R}_{\geq 0}^d \ | \ \forall \boldsymbol{\nu}\in \mathbb{R}^d\right\}. \label{mathcalA_Definition}
\end{equation}
We assume that the size of $\mathcal{A}$ is finite and denote it by $K$. This assumption is relatively mild since, for instance, in linear programming, the optimal solution can always be found at one of the vertices of the feasible region \citep{Princeton_LP}. \par
Also, let $\texttt{POSSIBLE-PI}(s) = \{ x \in \mathbb{R} \ | \ \exists \boldsymbol{\pi}\in \mathcal{A}, \pi_s = x \}$, where $\pi_s$ denote the $s$-th element of $\boldsymbol{\pi}$. We can see \texttt{POSSIBLE-PI($s$)} as the set of all possible values that an action in the action set can take as the $s$-th element. Let us denote the size of $\texttt{POSSIBLE-PI}(s)$ by $B_s$. Note that $K$, the size of $\mathcal{A}$, can be exponentially large in $d$, i.e., $|\mathcal{A}| = \mathcal{O}(\prod_{s =1}^{d} B_s)$. \par
We assume we have \emph{offline oracles} which efficiently solve the linear optimization problem (\ref{Abstract_Formulation_IntegerProgramminProblem}) once $\nu$ is given. For instance, algorithms which output $\boldsymbol{\pi}^{*}({\boldsymbol{\nu}})~=~\argmax_{\boldsymbol{\pi} \in \mathcal{A}} \boldsymbol{\nu}^{\top}\boldsymbol{\pi}$ in polynomial or pseudo-polynomial \footnote{A pseudo-polynomial time algorithm is a numeric algorithm whose running time is polynomial in the numeric value of the input, but not necessarily in the length of the input \citep{Garey1990}} time in $d$. \par
The player's objective is to identify the best action $\boldsymbol{\pi}^{*} = \argmax_{\boldsymbol{\pi}\in\mathcal{A}} \boldsymbol{\mu}^{\top}\boldsymbol{\pi}$ by pulling a single arm each round. The player is given a \emph{budget} $T$, and cannot pull arms more than $T$ times. The player outputs an action $\boldsymbol{\pi}^{\mathrm{out}}$ at the end, and she is evaluated by the \emph{probability of error}, which is formally $\Pr\left[ \boldsymbol{\pi}^{\mathrm{out}} \neq \boldsymbol{\pi}^{*} \right]$.

\section{Lower Bound of the Fixed-Budget R-CPE-MAB} \label{Lower_Bound_Section}
In this section, we show a lower bound of the probability of error in R-CPE-MAB. As preliminaries, let us introduce some notions.
First, we introduce $\boldsymbol{\pi}^{(s)}$ as follows \citep{Nakamura_TS_FixedConf2023}:
\begin{equation}
    \boldsymbol{\pi}^{(s)} = \argmin_{\boldsymbol{\pi} \in \mathcal{A} \setminus \left\{ \boldsymbol{\pi}^{*} \right\}} \frac{ \boldsymbol{\mu}^{\top} \left(
\boldsymbol{\pi}^{*} - \boldsymbol{\pi} \right)}{\left| \pi^{*}_{s} - \pi_s \right|}.
\end{equation}
Intuitively, among the actions whose $s$-th element is different from $\boldsymbol{\pi}^{*}$, $\boldsymbol{\pi}^{(s)}$ is the action which is the most difficult to determine whether it is the best action or not.
We also introduce the notion \emph{G-gap} \citep{Nakamura_TS_FixedConf2023} as follows:
\begin{eqnarray}
    \Delta_{s} & = & \min_{{\boldsymbol{\pi} \in \mathcal{A} \setminus \left\{ \boldsymbol{\pi}^{*} \right\}}} \frac{ \boldsymbol{\mu}^{\top} \left(
\boldsymbol{\pi}^{*} - \boldsymbol{\pi} \right)}{\left| \pi^{*}_{s} - \pi_s \right|} \nonumber \\
    & = & \frac{ \boldsymbol{\mu}^{\top} \left( \boldsymbol{\pi}^{*} - \boldsymbol{\pi}^{(s)} \right)}{\left| \pi_{s}^{*} - \pi^{(s)}_s \right|}. 
\end{eqnarray}
\emph{G-Gap} was first introduced in \citet{Nakamura_TS_FixedConf2023} as a natural generalization of the notion \emph{gap} in the CPE-MAB literature \citep{SChen2014,LChen2016,LChen2017}. This was introduced as a key notion that characterizes the difficulty of the problem instance. \par
In Theorem \ref{thm:LowerBoundTheorem}, we show that the sum of the inverse of squared \emph{G-Gaps}, 
\begin{eqnarray}
    \mathbf{H} = \sum_{s = 1}^{d} \left(\frac{1}{\Delta_s} \right)^2,
\end{eqnarray}
appears in the lower bound of the probability of error of R-CPE-MAB, which implies that it characterizes the difficulty of the problem instance.
\begin{restatable}[]{theorem}{LowerBoundTheorem} \label{thm:LowerBoundTheorem}
    For any action class and any algorithm that returns an action $\boldsymbol{\pi}^{\mathrm{out}}$ after $T$ times of arm pulls, the probability of error is at least
    \begin{equation}  
    \mathcal{O}\left( \exp\left( - \frac{T}{\mathbf{H}}  \right) \right). \label{LowerBoundEq}
    \end{equation}
\end{restatable}
We show the proof in Appendix \ref{LowerBoundTheoremProof}. If $\mathcal{A}$ is a set of $d$ dimensional standard basis, R-CPE-MAB becomes the standard best arm identification problem whose objective is to identify the best arm with the largest expected reward among $d$ arms \citep{Bubeck2009,Audibert2010,Carpentier2016}. From \citet{Carpentier2016}, a lower bound is $\mathcal{O}\left(\exp\left( - \frac{T}{\log (d) \mathbf{H}} \right)\right)$ for the standard best arm identification problem. It is a future work that whether the lower bound is $\mathcal{O}\left(\exp\left( - \frac{T}{\log (d) \mathbf{H}} \right)\right)$ for general action classes.

\section{The Combinatorial Successive Assign (CSA) Algorithm} \label{CSA_Algorithm_Section}
In this section, we first introduce the CSA algorithm, which can be seen as a generalization of the CSAR algorithm \citep{SChen2014}. This algorithm can be applied to fixed-budget R-CPE-MAB even when the size of the action class $\mathcal{A}$ is exponentially large in $d$. Then, we show an upper bound of the probability of error of the best action of the CSA algorithm. We also discuss the number of times offline oracles have to be called. 
% Finally, we theoretically compare the upper bounds of the probability of error of the CSA algorithm with the Minimax-CombSAR algorithm and the Peace algorithm.
\subsection{CSA Algorithm}
In this subsection, we introduce the CSA algorithm, a fully parameter-free algorithm for fixed-budget R-CPE-MAB that works even when the action set $\mathcal{A}$ can be exponentially large in $d$. \par 
We first define the \emph{constrained offline oracle} which is used in the CSA algorithm.
\begin{definition}[Constrained offline oracle]
Let $\boldsymbol{S} = \{ (e, x) | (e, x) \in \mathbb{Z} \times \mathbb{R} \}$ be a set of tuples. A constrained offline oracle is denoted by $\mathrm{COracle}$: $\mathbb{R}^{d} \times \boldsymbol{S} \rightarrow \mathcal{A}\ \cup \perp$ and satisfies
\begin{equation}
     \mathrm{COracle}(\boldsymbol{\mu}, \boldsymbol{S}) = \left\{
            \begin{array}{ll}
            \argmax_{\boldsymbol{\pi}\in\mathcal{A}_{\boldsymbol{S}}} \boldsymbol{\mu}^{\top} \boldsymbol{\pi} & (\mathrm{if} \ \mathcal{A}_{\boldsymbol{S}} \neq \emptyset), \\
            \perp & (\mathrm{if} \ \mathcal{A}_{\boldsymbol{S}} = \emptyset), \\
            \end{array} \nonumber
              \right.
\end{equation}
where we define $\mathcal{A}_{\boldsymbol{S}} = \left\{ \boldsymbol{\pi}\in\mathcal{A} \ | \ \forall (e, x) \in \boldsymbol{S}, \pi_e = x  \right\}$ as the collection of feasible actions and $\perp$ is a null symbol.
\end{definition}
Here, we can see that a COracle is a modification of an offline oracle specified by $\boldsymbol{S}$. In other words, for all $(e, x)$ in $\boldsymbol{S}$, the COracle outputs an action whose $e$-th element is $x$; otherwise, it outputs the null symbol. In Appendix \ref{COracleConstructionAppendix}, we discuss how to construct such COracles for some combinatorial problems, such as the optimal transport problem and the knapsack problem. \par
We introduce the CSA algorithm in Algorithm \ref{CSAAlgorithm}. The CSA algorithm divides the budget into $d$ rounds. In each round, we pull each of the remaining arms the same number of times (line \ref{pull_arms_line}). At each round $t$, the CSA outputs the empirically best action $\hat{\boldsymbol{\pi}}(t)$ (line \ref{pi_hat}), chooses a single arm $p(t)$ (line \ref{choose_p(t)}), and assigns $\hat{\pi}_{p(t)}(t)$ for the $e$-th element of $\boldsymbol{\pi}^{\mathrm{out}}$. Indices that are assigned are maintained in $\boldsymbol{F}(t)$, and arms $s \in F(t)$ will no longer be pulled in the next rounds. The pair of the index and the assigned value for $\boldsymbol{\pi}^{\mathrm{out}}$, $\boldsymbol{S}(t) $, is updated at every round (line \ref{S_t_update}).

\begin{algorithm}[t]
    \caption{CSA: Combinatorial Successive Assign Algorithm}
    \begin{algorithmic}[1] \label{CSAAlgorithm}
         \renewcommand{\algorithmicrequire}{\textbf{Input:}}
         \renewcommand{\algorithmicensure}{\textbf{Parameter:}}
         \REQUIRE Budget: $T\geq0$; COracle: $\rightarrow \mathcal{A}\cup \{ \perp \}$
         % \ENSURE  A strategy $S$ to select $\boldsymbol{\pi}^{t}$ every time step
         % \STATE \texttt{// Initialization} 
         \STATE Define $\Tilde{\log}(n) = \sum_{i = 1}^{d} \frac{1}{i}$
         \STATE $\Tilde{T}_{0} \leftarrow 0$, $\boldsymbol{F}(t) \leftarrow \emptyset$, $\boldsymbol{S}(t) = \emptyset$ 
         \FOR{$t = 1 \ \mathrm{to} \ d$}
            \STATE $\Tilde{T}(t) \leftarrow \lceil \frac{T - d}{\Tilde{\log}(d) (d - t + 1)} \rceil$
            \STATE Pull each arm $e\in[d] \setminus F(t) $ for $\Tilde{T}(t) - \Tilde{T}_{t - 1}$ times \label{pull_arms_line}
            \STATE $\hat{\boldsymbol{\pi}}(t) \leftarrow \mathrm{COracle}(\hat{\boldsymbol{\mu}}(t), \boldsymbol{S}(t))$ \label{pi_hat}
            \IF{$\hat{\boldsymbol{\pi}}(t) = \perp$}
                \STATE $\textbf{Fail}$: set ${\boldsymbol{\pi}}^{\mathrm{out}}\leftarrow \perp$ and return ${\boldsymbol{\pi}}^{\mathrm{out}}$
            \ENDIF
            \FORALL{ $e$ in $[d]\setminus \boldsymbol{F}(t)$}
            \STATE $\texttt{CHECK(e)} \leftarrow \texttt{POSSIBLE-PI(e)} \setminus \left\{ \hat{\pi}_e(t) \right\}$
            \STATE $\Tilde{\boldsymbol{\pi}}^{e}(t) \leftarrow \argmax_{x \in \texttt{CHECK(e)} } \mathrm{COracle}(\hat{\boldsymbol{\mu}}(t), \boldsymbol{S}(t) \cup (e, x))$ \label{COracleline}
            \ENDFOR
        \STATE $ p(t) \leftarrow \argmax_{e \in [d] \setminus F(t)} \frac{\langle \hat{\boldsymbol{\mu}}(t), \hat{\boldsymbol{\pi}}(t) - \Tilde{\boldsymbol{\pi}}^{e}(t) \rangle}{\hat{\pi}_{e}(t) - \tilde{\pi}^{e}_{e}(t)}$ \label{choose_p(t)}
        \STATE $\boldsymbol{F}(t + 1) \leftarrow \boldsymbol{F}(t) \cup \{p(t) \}$
        \STATE $\boldsymbol{S}(t + 1)\leftarrow \boldsymbol{S}(t) \cup \{ (p(t), \hat{\pi}_{p(t)}(t) \}$ \label{S_t_update}
        \ENDFOR
        \STATE \texttt{//  Convert $\boldsymbol{S}(d + 1)$ to $\boldsymbol{\pi}^{\mathrm{out}}$}
        \FOR{$(e, x) \ \mathrm{in} \ \boldsymbol{S}(d + 1)$}
            \STATE \texttt{//  The $e$-th element of $\boldsymbol{\pi}^{\mathrm{out}}$ is $x$}
            \STATE $\pi^{\mathrm{out}}_e = x$
        \ENDFOR
        \RETURN ${\boldsymbol{\pi}}^{\mathrm{out}}$
    \end{algorithmic} 
\end{algorithm}
\subsection{Theoretical Analysis of CSA Algorithm}
Here, we first discuss an upper bound of the probability of error. Then, we discuss the number of times we call the offline oracle. 
\subsubsection{An Upper Bound of the Probability of Error}
Here,  we show an upper bound of the probability of error of the CSA algorithm. Let $\Delta_{(1)}, \ldots, \Delta_{(d)}$ be a permutation of $\Delta_{1}, \ldots, \Delta_{d}$ such that $\Delta_{(1)} \leq \cdots \leq \Delta_{(d)}$. Also, let us define 
\begin{eqnarray}
    \mathbf{H}_{2} = \max_{ s\in [d]} \frac{s}{\Delta^{2}_{(s)}}. \nonumber
\end{eqnarray}
One can verify that $\mathbf{H}_2$ is equivalent to $\mathbf{H}$ up to a logarithmic factor: $\mathbf{H}_{2} \leq \mathbf{H} \leq \log(2d)\mathbf{H}_{2}$ \citep{Audibert2010}. \par

\begin{restatable}[]{theorem}{CSATheorem} \label{CSATheorem}
    Given any $T > d$, action class $\mathcal{A} \subset \mathbb{R}^{d}$, and $\boldsymbol{\mu} \in \mathbb{R}^{d}$, the CSA algorithm uses at most $T$ samples and outputs a solution $\boldsymbol{\pi}^{\mathrm{out}} \in \mathcal{A} \cup \{\perp \}$ such that
    \begin{eqnarray}\label{CSATheoremUpperBound}
        &&\Pr\left[ \boldsymbol{\pi}^{\mathrm{out}} \neq \boldsymbol{\pi}^{*} \right] \nonumber \\ 
        &\leq&  d^2 \exp \left( - \frac{T - d}{2(2 + L^2)^2 R^2\Tilde{\log}(d)U_{\mathcal{A}}^2 \mathbf{H}_2} \right),
    \end{eqnarray}
    where $L = \max_{ e\in[d], \boldsymbol{\pi}^{1}, \boldsymbol{\pi}^{2}, \boldsymbol{\pi}^{3} \in \mathcal{A}, \pi^{1}_{e} \neq \pi^{3}_{e}} \frac{\left| \pi^{1}_e - \pi^{2}_e \right|}{\left| \pi^{1}_{e} - \pi^{3}_{e} \right|}$, $\Tilde{\log}(d) \triangleq \sum_{s = 1}^d \frac{1}{s}$, and $U_{\mathcal{A}} = \max_{\boldsymbol{\pi}, \boldsymbol{\pi}' \in \mathcal{A}, e \in \{ s \in [d] \ | \ \pi_s \neq \pi'_s  \} } \frac{\sum_{s = 1}^{d} | \pi_s - \pi'_s |}{ | \pi_e - \pi'_e |}$.
\end{restatable}
We can see that the CSA algorithm is optimal up to a logarithmic factor in the exponent.
Since $L = 1$ and $U_{\mathcal{A}} = \mathrm{width}(\mathcal{A})$ in the ordinary CPE-MAB, we can confirm that Theorem \ref{CSATheorem} can be seen as a natural generalization of Theorem 3 in \citet{SChen2014}, which shows an upper bound of the probability of error of the CSAR algorithm. 
\subsubsection{The Oracle Complexity}
Next, we discuss the \emph{oracle complexity}, the number of times we call the offline oracle \citep{ITO2019,Xu2021}. Note that, in the CSA algorithm, we call the COracle $\mathcal{O}(d\sum_{s = 1}^{d}B_s)$ times (line \ref{COracleline}). Therefore, if each COracle call invokes the offline oracle $N$ times, the \emph{oracle complexity} is $\mathcal{O}\left(Nd\sum_{s = 1}^{d}B_s\right)$. Finally, if the time complexity of each oracle call is $Z$, then the total time complexity of the CSA algorithm is $\mathcal{O}\left(ZNd\sum_{s = 1}^{d}B_s\right)$. \par
Below, we discuss the oracle complexity of the CSA algorithm in some specific combinatorial problems such as the knapsack problem and the optimal transport problem. Note that the Minimax-CombSAR algorithm and Peace algorithm \citet{Katz-Samuels2020} have to enumerate all the actions in advance, where the number may be exponentially large in $d$. We show that the CSA algorithm mitigates the curse of dimensionality. \par
\textbf{The Knapsack Problem \citep{Dantzig2007}.}  In the knapsack problem, we have $d$ items. Each item $s\in[d]$ has a weight $w_s$ and value $v_s$. Also, there is a knapsack whose capacity is $W$ in which we put items. Our goal is to maximize the total value of the knapsack, not letting the total weight of the items exceed the capacity of the knapsack. Formally, the optimization problem is given as follows:
\begin{equation*} \label{knapsackproblem_formulation}
\begin{array}{ll@{}ll}
\text{maximize}_{\boldsymbol{\boldsymbol{\pi}}\in\mathcal{A}}  & \sum_{s = 1}^{d}v_{s}\pi_s  &\\ \\
\text{subject to}& \sum_{s = 1}^{d}\pi_s w_s \leq W, &
\end{array}
\end{equation*}
where $\pi_s \in \mathbb{Z}_{\geq 0}$ denotes the number of item $s$ in the knapsack. Here, if we assume the weight of each item is known, but the value is unknown, we can apply the R-CPE-MAB framework to the knapsack problem, where we estimate the values of items. The knapsack problem is NP-complete \citep{Garey1979}. Hence, it is unlikely that the knapsack problem can be solved in polynomial time. However, it is well known that the knapsack problem can be solved in pseudo-polynomial time $\mathcal{O}(dW)$ if we use dynamic programming \citep{Kellerer2011,Fujimoto2016}. It finds the optimal solution by constructing a table of size $dW$ whose $(s, w)$-th element represents the maximum total value that can be achieved if the sum of the weights does not exceed $w$ using up to the $s$th item.  In some cases, it is sufficient to assume $\mathcal{O}(dW)$ time-complexity is enough, and therefore, we use this dynamic programming method as the offline oracle. We can construct the \emph{COrcale} for the knapsack problem by calling this offline oracle once (see Appendix \ref{COracleConstructionAppendix_KnapsackProblem} for details). 
Therefore, the CSA algorithm calls the offline oracle $\mathcal{O}(1 \times d\sum_{s = 1}^{d}B_s ) = \mathcal{O}(d\sum_{s = 1}^{d}B_s)$ times, and the total time complexity of the CSA algorithm is $\mathcal{O}(dW \times d\sum_{s = 1}^{d}B_s ) = \mathcal{O}(d^2\sum_{s = 1}^{d}B_sW)$. This is much more computationally friendly than the Peace algorithm with time complexity $\mathcal{O}\left( \prod_{s = 1}^{d}B_s \right)$.\par
\textbf{The Optimal Transport (OT) Problem \citep{PeyreCuturi2019}.} OT can be regarded as the cheapest plan to deliver resources from $m$ suppliers to $n$ consumers, where each supplier $i$ and consumer $j$ have supply $s_i$ and demand $d_j$, respectively. Let $\boldsymbol{\gamma}\in\mathbb{R}^{m \times n}_{\geq 0}$ be the cost matrix, where $\gamma_{ij}$ denotes the cost between supplier $i$ and demander $j$. Our objective is to find the optimal transportation matrix
\begin{equation} \label{OT_formulation}
    \boldsymbol{\pi}^{*} = \argmin_{\boldsymbol{\pi}\in \mathcal{G}(\boldsymbol{s}, \boldsymbol{d})} \sum_{i, j}\pi_{ij}\gamma_{ij},
\end{equation}
where 
\begin{align}\label{coupling_constraint}
    \mathcal{G}(\boldsymbol{s}, \boldsymbol{d}) \triangleq \left\{ \boldsymbol{\Pi} \in \mathbb{R}^{m \times n}_{\geq 0} \,\middle|\, \boldsymbol{\Pi} \boldsymbol{1}_n = \boldsymbol{s}, \boldsymbol{\Pi}^{\top} \boldsymbol{1}_m = \boldsymbol{d} \right\}.
\end{align}
Here, $\boldsymbol{s} = (s_1, \ldots, s_m)$ and $\boldsymbol{d} = (d_1, \ldots, d_n)$. 
% Note that the dimension $d$ in (\ref{Abstract_Formulation_IntegerProgramminProblem}) is $mn$ here. 
$\pi_{ij}$ represents how much resources one sends from supplier $i$ to demander $j$. 
Here, if we assume that the cost is unknown and changes stochastically, e.g., due to some traffic congestions, we can apply the R-CPE-MAB framework to the optimal transport problem, where we estimate the cost of each edge $\left(i, j\right)$ between supplier $i$ and consumer $j$. \par
Once $\boldsymbol{\gamma}$ is given, we can compute $\boldsymbol{\pi}^{*}$ in $\mathcal{O}(l^3\log l)$, where $l = \max\left(m, n\right)$, by using network simplex or interior point methods \citep{Cuturi2013}, and we can use them as the offline oracle.
It is known that the solution of linear programming can always be found at one of the vertices of the feasible region \citep{Princeton_LP}, and therefore the size of the action space $\mathcal{A} = \left\{ \boldsymbol{\pi}^{\boldsymbol{\nu, \mathcal{G}(\boldsymbol{s}, \boldsymbol{d})}} \in \mathbb{R}_{\geq 0}^{m\times n} \ | \ \forall \boldsymbol{\nu}\in \mathbb{R}^{m \times n}\right\}$ is finite. However, it is difficult to construct $\left\{\texttt{POSSIBLE-PI}\left(\left(i, j\right)\right)\right\}_{i \in [m], j \in [n]}$ in general to run the CSA algorithm.
On the other hand, if $\boldsymbol{s}$ and $\boldsymbol{d}$ are both integer vectors, we can construct $\left\{\texttt{POSSIBLE-PI}\left(\left(i, j\right)\right)\right\}_{i \in [m], j \in [n]}$ thanks to the fact that $\mathcal{A} = \left\{ \boldsymbol{\pi}^{\boldsymbol{\nu, \mathcal{G}(\boldsymbol{s}, \boldsymbol{d})}} \in \mathbb{Z}_{\geq 0}^{m \times n} \ | \ \forall \boldsymbol{\nu}\in \mathbb{R}^{m \times n}\right\}$ is a set of non-negative integer matrices \citep{Goodman1997}. In this case, all actions are restricted to non-negative integers, and $\texttt{POSSIBLE-PI}(i, j) = \left\{ 0, 1, \ldots, \min\left( s_{i}, d_{j} \right) \right\}$.
In Appendix \ref{COracleConstructionAppendix_OptimalTransport}, we show that the COracle can be constructed by calling the offline oracle once. Therefore, we call the offline oracle $\mathcal{O}\left( 1 \times mn\sum_{s = 1}^{mn}B_s \right) = \mathcal{O}(mn\sum_{s = 1}^{mn} B_s)$ times, and the total time complexity is $\mathcal{O}\left( l^{3} \log l mn\sum_{s = 1}^{mn} B_s \right)$, where $l = \max \left( m, n \right)$. Again, this is much more computationally friendly than the Peace algorithm with time complexity $\mathcal{O}\left( \prod_{s = 1}^{d} B_{d} \right)$. \par
\textbf{A General Case When $K (=|\mathcal{A}|)  = \mathrm{poly}(d).$} In some cases, we can enumerate all the possible actions in $\mathcal{A}$.  For instance, one may use some prior knowledge of each arm, which is sometimes obtainable in the real world, to narrow down the list of actions, and make the size of the action class $\mathcal{A}$ polynomial in $d$. 
% In this case, finding the best action $\boldsymbol{\pi}^{*}$ from $\mathcal{A}$ is $\mathcal{O}(K \log K)$ since we can compute $\boldsymbol{\nu}^{\top}\boldsymbol{\pi}$ for all $\boldsymbol{\pi}\in\mathcal{A}$, and sort actions with respect to the value $\boldsymbol{\nu}^{\top}\boldsymbol{\pi}$. 
In Appendix \ref{COracleConstructionAppendix_GeneralCase}, 
% we show that the COracle for this case can be constructed by calling this offline oracle once. Therefore, the oracle complexity is $\mathcal{O}(1 \times d\sum_{s = 1}^{d}B_s ) = \mathcal{O}(d\sum_{s = 1}^{d}B_s)$, 
we show that the time complexity of the COracle is $\mathcal{O}\left( dK + K \log K \right)$, and therefore the total time complexity of the CSA algorithm is $\mathcal{O}\left((dK + K \log K) \cdot d\sum_{s = 1}^{d}B_s \right)$.

\section{The Minimax-CombSAR Algorithm for R-CPE-MAB where $|\mathcal{A}| = \mathrm{poly}(d)$} \label{Minimax-CombSAR_Algorithm_Section}
In this section, we show an algorithm for fixed-budget R-CPE-MAB named the Minimax Combinatorial Successive Accept (Minimax-CombSAR) algorithm, for the case where we can assume that the size of $\mathcal{A}$ is polynomial in $d$. Let $\mathcal{A} ~=~ \{ \boldsymbol{\pi}^{1} ~=~ \boldsymbol{\pi}^{*}, \boldsymbol{\pi}^{2}, \ldots, \boldsymbol{\pi}^{K} \}$, where $\boldsymbol{\mu}^{\top} \boldsymbol{\pi}^{i} \geq \boldsymbol{\mu}^{\top}\boldsymbol{\pi}^{i + 1}$, for all $i \in [K - 1]$.  The Minimax-CombSAR algorithm is inspired by the  Optimal Design-based Linear Best Arm Identification (OD-LinBAI) algorithm \cite{Yang2022}, which eliminates actions in order from those considered suboptimal and finally outputs the remaining action as the optimal action. For the Minimax-CombSAR algorithm, we show an upper bound of the probability of error.
\subsection{Minimax-CombSAR Algorithm}
We show the Minimax-CombSAR in Algorithm \ref{Minimax-CombSAR_Algorithm}. Here, we explain it at a higher level. \par
We have $\lceil \log d \rceil$ phases in the Minimax-CombSAR algorithm, and it maintains an \emph{active} action set $\mathbb{A}(r)$ in each phase $r$. In each phase $r \in [\lceil \log d \rceil]$, it pulls arms $m(r) = \frac{T'- d\lceil \log_2 d \rceil}{B/2^{r - 1}}$ times in total, where $B = 2^{ \left\lceil \log_2 d \right \rceil} - 1$ and $\beta \in [0, 1]$ is a hyperparameter, and $T' = T - \left\lfloor \frac{T}{d}\beta \right\rfloor \times d$. 
In each phase, we compute an \emph{allocation vector} $\boldsymbol{p}(r) \in \left\{ \boldsymbol{v}\in \mathbb{R}^d \ | \ \sum_{s = 1}^{d} v_s = 1\right\} \triangleq \Pi_d$, and pull each arm $s$ $ \lceil p_s(r)\cdot m(r) \rceil$ times. Then, at the end of each phase $r$, it eliminates a subset of possibly suboptimal actions. Eventually, there is only one action $\boldsymbol{\pi}^{\mathrm{out}}$ in the active action set, which is the output of the algorithm. \par
The key to identifying the best action with high confidence is the choice of the allocation vector $\boldsymbol{p}(r)$, which determines how many times we pull each arm in phase~$r$. \par
\begin{algorithm}[t]
    \caption{Minimax Combinatorial Successive Accept Algorithm}
    \begin{algorithmic}[1] \label{Minimax-CombSAR_Algorithm}
         \renewcommand{\algorithmicrequire}{\textbf{Input:}}
         \renewcommand{\algorithmicensure}{\textbf{Parameter:}}
         \REQUIRE Budget: $T\geq0$, initialization parameter: $\beta$, action set: $\mathbb{A}(1) = \mathcal{A}$
         % \ENSURE  A strategy $S$ to select $\boldsymbol{\pi}^{t}$ every time step
         \STATE \texttt{// Initialization} 
          \FOR{$s\in[d]$} \label{Initialization_start}
            \STATE Pull arm $s$ $\left\lfloor \frac{T}{d}\beta \right \rfloor$ times and update $\hat{\mu}_{s}(1)$
            \STATE $T_s(r) \leftarrow \left\lfloor \frac{T}{d}\beta \right \rfloor$
          \ENDFOR \label{Initialization_end}
          \STATE  $T' \leftarrow T - \left\lfloor \frac{T}{d}\beta \right \rfloor \times d$
          \FOR{$r = 1 \ \mathrm{to} \ \lceil \log_{2} d \rceil$}
            \STATE $m(r) = \frac{T'- d\lceil \log_2 d \rceil}{B/2^{r - 1}}$
            \STATE Compute $\boldsymbol{p}(r)$ according to (\ref{min-max-p(r)-ordinary}) or (\ref{min-max-p(r)_Lagrange})
            \FOR{$s\in [d]$}
                \STATE Pull arm $s$ $\left\lceil p_s(r) \cdot m(r) \right\rceil$ times
                \STATE Update $\hat{\mu}_{s}(r + 1)$ with the observed samples
            \ENDFOR
            \STATE For each action $\boldsymbol{\pi} \in \mathbb{A}(r)$, estimate the expected reward: $ \langle \hat{\boldsymbol{\mu}}(r + 1), \boldsymbol{\pi} \rangle$
            \STATE Let $\mathbb{A}(r + 1)$ be the set of $\left\lceil\frac{d}{2^{r}}\right\rceil$ actions in $\mathbb{A}(r)$ with the largest estimates of the expected rewards \label{AcceptEliminationProcedure}
          \ENDFOR
          \renewcommand{\algorithmicrequire}{\textbf{Output:}}
          \REQUIRE The only action $\boldsymbol{\pi}^{\mathrm{out}}$ in $\mathbb{A}(\lceil\log_2 d\rceil + 1)$
 \end{algorithmic} 
\end{algorithm}
\textbf{Choice of $\boldsymbol{p}(r)$} \par
We discuss how to choose an allocation vector that is beneficial to identify the best action. Let us denote the number of times arm $s$ is pulled before phase $r$ starts by $T_s(r)$. Also, we denote by $\hat{\mu}_s(r)$ the empirical mean of the reward from arm $s$ before phase $r$ starts. Then, at the end of phase $r$, from Hoeffding's inequality \citep{Hoeffding1963}, we have
\begin{eqnarray} \label{HoeffdingInequality}
    &&\Pr\left[ \left|(\hat{\boldsymbol{\mu}}(r + 1) -\boldsymbol{\mu})^{\top} (\boldsymbol{\pi}^{a} - \boldsymbol{\pi}^{b})\right| \geq \epsilon \right] \nonumber \\ 
    &\leq& \exp \left( - \frac{  \epsilon^2}{ \kappa^{a, b, \boldsymbol{p}}(r) R^2}  \right) 
\end{eqnarray}
for any $\boldsymbol{\pi}^{a}, \boldsymbol{\pi}^{b} \in \mathcal{A}$, and $\epsilon \in \mathbb{R}$. Here, $\hat{\boldsymbol{\mu}}(r + 1) = (\hat{\mu}_1(r + 1), \ldots, \hat{\mu}_d(r + 1))^{\top}$ and 
\begin{equation}
    \kappa^{a, b, \boldsymbol{p}}(r) = \sum_{s = 1}^{d} \frac{\left( \pi^{a}_{s} - \pi^{b}_{s} \right)^2}{T_s(r) + \lceil p_s(r)\cdot m(r) \rceil }.
\end{equation}
(\ref{HoeffdingInequality}) shows that the empirical difference $\hat{\boldsymbol{\mu}}^{\top} (\boldsymbol{\pi}^{a} - \boldsymbol{\pi}^{b})$ between $\boldsymbol{\pi}^{a}$ and $\boldsymbol{\pi}^{b}$ is closer to the true difference  ${\boldsymbol{\mu}}^{\top} (\boldsymbol{\pi}^{a} - \boldsymbol{\pi}^{b})$ with high probability if we make $\kappa^{a, b, \boldsymbol{p}}(r)$ small. In that case, we have a higher chance to distinguish whether $\boldsymbol{\pi}^{a}$ is better than $\boldsymbol{\pi}^{b}$ or not.
However, when we have more than two actions in $\mathbb{A}(r)$, 
\begin{equation}
    \boldsymbol{p}^{a, b}(r) = \argmin_{\boldsymbol{p} \in \Pi_d} \kappa^{a, b, \boldsymbol{p}}(r) \nonumber
\end{equation}
is not necessarily a good allocation vector for investigating the true difference between other pairs of actions in $\mathbb{A}(r)$. Therefore, we propose the following allocation vector as an alternative:
\begin{equation}
    \boldsymbol{p}^{\mathrm{min}}(r) \triangleq \argmin_{\boldsymbol{p} \in \Pi_d } \max_{\boldsymbol{\pi}^{a}, \boldsymbol{\pi}^{b} \in \mathcal{A}} \kappa^{a, b, \boldsymbol{p}}(r). \label{min-max-p(r)-ordinary}
\end{equation}
(\ref{min-max-p(r)-ordinary}) takes a minimax approach, which computes an allocation vector that minimizes the maximum value of the right-hand side of (\ref{HoeffdingInequality}) among all of the pairs of actions in $\mathbb{A}(r)$. Since (\ref{min-max-p(r)-ordinary}) is a $d$-dimensional non-linear optimization problem, it becomes computationally costly as $d$ grows. Thus, another possible choice of the allocation vector is as follows:
\begin{eqnarray} 
    \boldsymbol{q}^{\mathrm{min}}(r) 
    &\triangleq& \argmin_{\boldsymbol{p} \in \Pi_d} \max_{\boldsymbol{\pi}^{a}, \boldsymbol{\pi}^{b} \in \mathcal{A}} \lambda^{a, b, \boldsymbol{p}} (r), \label{min-max-p(r)-2}
\end{eqnarray}
where $\lambda^{a, b, \boldsymbol{p}} = \sum_{s = 1}^{d} \frac{\left( \pi^{a}_{s} - \pi^{b}_{s} \right)^2}{p_s(r)\cdot m(r) }$.
For specific actions $\boldsymbol{\pi}^{k}, \boldsymbol{\pi}^{l} \in \mathcal{A}$, from the method of Lagrange multipliers \citep{Hoffmann2009}, we have
\begin{eqnarray}
    \boldsymbol{q}^{k, l}(r) 
    &\triangleq& \argmin_{\boldsymbol{p} \in \Pi_d} \lambda^{k, l, \boldsymbol{p}} (r) \nonumber \\
    &=& \left(\frac{\left| \pi^{k}_1 - \pi^{l}_1 \right|}{\sum_{s = 1}^{d} \left| \pi^{k}_s - \pi^{l}_s \right|}, \ldots, \frac{\left| \pi^{k}_d - \pi^{l}_d \right|}{\sum_{s = 1}^{d} \left| \pi^{k}_s - \pi^{l}_s \right|} \right)^{\top}, \nonumber
\end{eqnarray}
and therefore, $\boldsymbol{q}^{\mathrm{min}}(r)$ in (\ref{min-max-p(r)-2}) can be written explicitly as follows:
\begin{align}
    \boldsymbol{q}^{\mathrm{min}}(r) &= \boldsymbol{q}^{i, j}(r) \nonumber \\
    &= \left(\frac{\left| \pi^i_1 - \pi^{j}_1 \right|}{\sum_{s = 1}^{d} \left| \pi^{i}_s - \pi^{j}_s \right|}, \ldots, \frac{\left| \pi^{i}_d - \pi^{j}_d \right|}{\sum_{s = 1}^{d} \left| \pi^{i}_s - \pi^{j}_s \right|} \right)^\top, \label{min-max-p(r)_Lagrange}
\end{align} 
where $\boldsymbol{\pi}^{i}, \boldsymbol{\pi}^{j} = \argmax_{\boldsymbol{\pi}^{k}, \boldsymbol{\pi}^{l} \in \mathcal{A}} \lambda^{k, l, \boldsymbol{p}} (r)$. If we use (\ref{min-max-p(r)_Lagrange}) for the allocation vector instead of computing (\ref{min-max-p(r)-ordinary}), we do not have to solve a $d$-dimensional non-linear optimization problem, and thus is computationally more friendly. \par
In some cases, actions in $\mathcal{A}$ can be sparse, and $\boldsymbol{p}(1)$ can also be sparse. If $\boldsymbol{p}(1)$ is sparse, we may only have a few samples for some arms, and therefore accidentally eliminate the best action in the first phase in line \ref{AcceptEliminationProcedure} of Algorithm \ref{Minimax-CombSAR_Algorithm}. To cope with this problem, we have the \emph{initialization} phase (lines \ref{Initialization_start}--\ref{Initialization_end}) to pull each arm $\left\lfloor \frac{T}{d}\beta \right \rfloor$ times, where $\beta \in [0, 1]$ is a hyperparameter. Intuitively, $\beta$ represents how much of the total budget will be spent on the initialization phase. If $\beta$ is too small, we may accidentally eliminate the best action in the early phase, and if it is too large, we may not have enough budget to distinguish between the best action and the next best action.

\subsection{Theoretical Analysis of Minimax-CombSAR} 
Here, in Theorem \ref{Minimax-CombSAR_Algorithm_MainTheorem}, we show an upper bound of the probability of error of the Minimax-CombSAR algorithm.
\begin{theorem}\label{Minimax-CombSAR_Algorithm_MainTheorem}
    For any problem instance in fixed-budget R-CPE-MAB, the Minimax-CombSAR algorithm outputs an action $\boldsymbol{\pi}^{\mathrm{out}}$ satisfying
    \begin{align}
        &\Pr\left[ \boldsymbol{\pi}^{\mathrm{out}} \neq \boldsymbol{\pi}^{*} \right] \nonumber \\
        &\le \left( \frac{4K}{d} + 3 \log_2 d \right) \exp \left( - \frac{T' - \left\lceil \log_2 d \right\rceil}{R^2 V^2} \cdot \frac{1}{\mathbf{H}_{2}} \right), \label{Minimax-CombSARUpperBound}
    \end{align}
    where $V = \max_{\boldsymbol{\pi}^{i} \in \mathcal{A} \setminus \{ \boldsymbol{\pi}^{*} \}}  \left(\sum\limits_{s = 1}^{d} \frac{\left| \pi^{1}_{s} - \pi^{i}_{s} \right|}{ \left| \pi^{1}_{s(i)} - \pi^{i}_{s(i)} \right| } \right)^2$ and $s(i) = \argmax_{s \in [d]} \left| \pi^{1}_s - \pi^{i}_s \right|$.
\end{theorem} 
We can see that the Minimax-CombSAR algorithm is an optimal algorithm whose upper bound of the probability of error matches the lower bound shown in (\ref{LowerBoundEq}) since we have $\mathbf{H}_{2} \leq \mathbf{H} \leq \log(2d)\mathbf{H}_{2}$. 
\subsection{Comparison with \citet{Katz-Samuels2020}}
Here, we compare the Minimax-CombSAR algorithm with the Peace algorithm introduced in \citet{Katz-Samuels2020}. The upper bound of the Peace algorithm can be written as follows:
\begin{eqnarray} 
    &&\Pr\left[ \boldsymbol{\pi}^{\mathrm{out}} \neq \boldsymbol{\pi}^{*} \right] \nonumber \\
    &\leq& 2 \left\lceil \log \left( d \right) \right\rceil \exp \left( -\frac{T}{c' \log (d) \left( \gamma_{*} + \rho_{*} \right)} \right), \label{PeaceAlgorithmUpperBound}
\end{eqnarray}
where 
\begin{align}
    \gamma_{*} = \inf_{\boldsymbol{\lambda} \in \Pi_{d}} \mathbb{E}_{\boldsymbol{\eta} \sim \mathcal{N}\left(0, I \right) } \left[ \sup_{\boldsymbol{\pi} \in \mathcal{A} \setminus \left\{ \boldsymbol{\pi}^{*} \right\}} \frac{\left( \boldsymbol{\pi}^{*} - \boldsymbol{\pi} \right)^{\top} \boldsymbol{A}\left( \lambda \right)^{\frac{1}{2}} \boldsymbol{\eta}}{ \boldsymbol{\mu}^{\top} \left( \boldsymbol{\pi}^{*} - \boldsymbol{\pi} \right) } \right], \nonumber
\end{align}
\begin{eqnarray}
    \rho_{*} = \min_{\boldsymbol{\lambda} \in \Pi_{d}} \max_{\boldsymbol{\pi} \in \mathcal{A} \setminus \{\boldsymbol{\pi}^{*}\}} \frac{\sum_{s = 1}^{d} \frac{\left| \pi^{*}_{s} - \pi_{s} \right|^{2}}{\lambda_{s}}}{\Delta^{2}_{\boldsymbol{\pi}^{*}, \boldsymbol{\pi}}}, \nonumber
\end{eqnarray}
and $c'$ is a problem-dependent constant. Here, $\boldsymbol{A}\left(\boldsymbol{\lambda}\right)$ is a diagonal matrix whose $(i, i)$ element is $\lambda_{i}$. In general, it is not clear whether the upper bound of the Minimax-CombSAR algorithm in (\ref{Minimax-CombSAR_Algorithm}) is tighter than that of the Peace algorithm shown in (\ref{PeaceAlgorithmUpperBound}). In the experiment section, we compare the two algorithms numerically and show that our algorithm outperforms the Peace algorithm.
% where $c'$ is a problem-dependent constant, $W = \max_{\boldsymbol{\pi}, \boldsymbol{\pi}' \in \mathcal{A}} \sum_{s = 1}^{d} \left| \pi_s - \pi'_s \right|$, and $\Delta' =  \boldsymbol{\mu}^{\top}\boldsymbol{\pi}^{*} - \max_{ \boldsymbol{\pi} \in \mathcal{A}\setminus\{ \boldsymbol{\pi}^{*} \}} \boldsymbol{\mu}^{\top}\boldsymbol{\pi}$. Since, $\frac{\mathrm{W}^{2}}{{\Delta'}^{2}} \geq \frac{1}{\Delta^{2}_{s}}$ for any $s \in [d]$, we have $d\frac{\mathrm{W}^{2}}{{\Delta'}^{2}} \geq \sum_{s = 1}^{d} \frac{1}{\Delta^{2}_{s}} = \mathbf{H} \geq \mathbf{H}_{2}$, and therefore, the upper bounds of the CSA and Minimax-CombSAR algorithms are tighter than the upper bound of the Peace algorithm.

\section{Experiment}\label{ExperimentSection}

In this section, we numerically compare the CSA, Minimax-CombSAR, and Peace algorithms. Here, the goal is to identify the best action for the knapsack problem. 
In the knapsack problem, we have $d$ items. Each item $s\in[d]$ has a weight $w_s$ and value $v_s$. Also, there is a knapsack whose capacity is $W$ in which we put items. Our goal is to maximize the total value of the knapsack, not letting the total weight of the items exceed the capacity of the knapsack. Formally, the optimization problem is given as follows:
\begin{equation*}
\begin{array}{ll@{}ll}
\text{maximize}_{\boldsymbol{\boldsymbol{\pi}}\in\mathcal{A}}  & \sum_{s = 1}^{d}v_{s}\pi_s  &\\ \\
\text{subject to}& \sum_{s = 1}^{d}\pi_s w_s \leq W, &
\end{array}
\end{equation*}
where $\pi_s$ denotes the number of item $s$ in the knapsack. Here, the weight of each item is known, but the value is unknown, and therefore has to be estimated. In each time step, the player chooses an item $s$ and gets an observation of value $v_s$, which can be regarded as a random variable from an unknown distribution with mean $v_s$. \par
First, we run an experiment where we assume $|\mathcal{A}|$ is exponentially large in $d$. We see the performance of the CSA algorithm with different budgets. Next, we compare the CSA, Minimax-CombSAR, and Peace algorithms, where we can assume that $|\mathcal{A}|$ is polynomial in $d$. 

\subsection{When $|\mathcal{A}|$ is Exponentially Large in $d$} 
\begin{figure}[t]
    \centering
    \includegraphics[width = \linewidth]{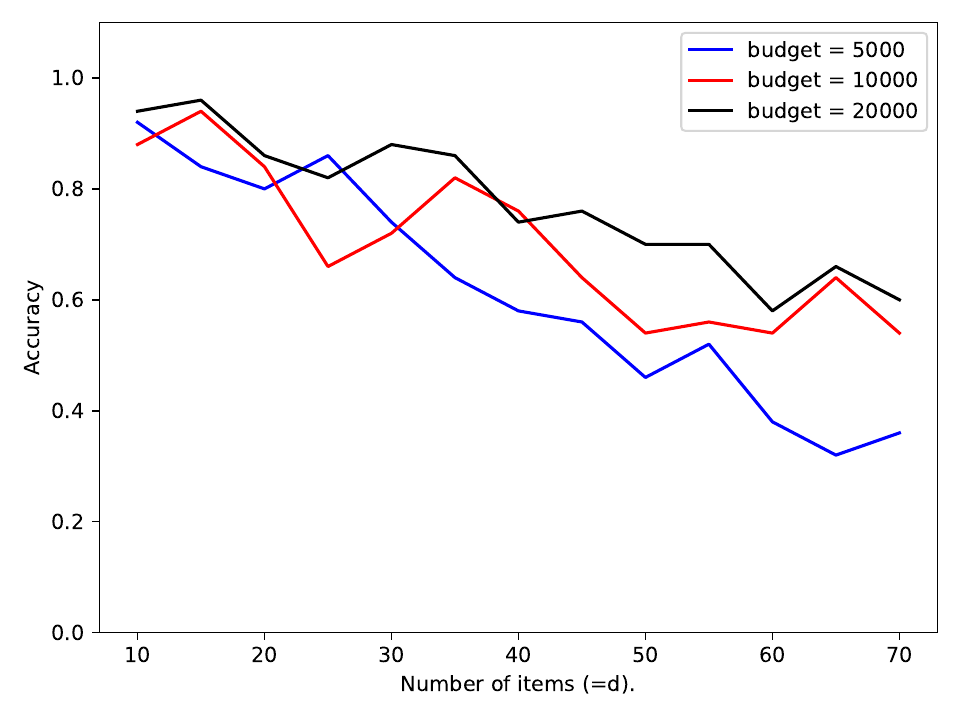}
    \caption{The percentage of the best actions correctly identified by the CSA algorithm in the knapsack problem with different budgets. The horizontal axis represents the number of items $d$. We ran the experiments fifty times for each $d$.}
    \label{experiment1}
\end{figure}
\begin{figure}[t]
    \centering
    \includegraphics[width = \linewidth]{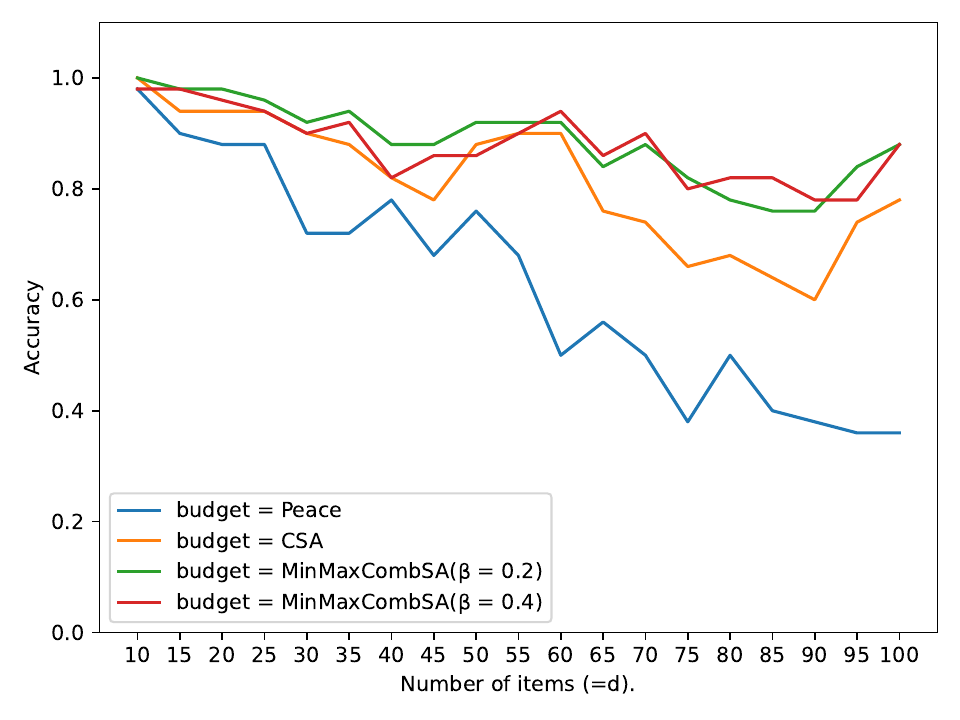}
    \caption{Comparison of the percentage of the best actions correctly identified by the CSA, Minimax-CombSAR, and Peace algorithms. The horizontal axis represents the number of items $d$. We ran the experiments fifty times for each $d$.}
    \label{experiment2}
\end{figure}
Here, we assume $|\mathcal{A}|$ is exponentially large in $d$. Therefore, we can not use the Minimax-CombSAR and the Peace algorithm since they have to enumerate all the possible actions. Here, we see how many times out of 50 experiments the CSA algorithm identified the best action.  \par
For our experiment, we generated the weight of each item $s$, $w_s$, uniformly from $\{1, 2, \ldots, 200 \}$. For each item $s$, we generated $v_s$ as $v_s = w_s \times x$, where $x$ is a uniform sample from $[1.0, 1.1]$. The capacity of the knapsack was $W = 200$. Each time we chose an item $s$, we observed a value $v_s + \epsilon$ where $\epsilon$ is a noise from $\mathcal{N}(0, 1)$. We set $R=1$. We show the result in Figure \ref{experiment1}. The larger the dimension $d$ is, and the smaller the budget $T$ is, the fewer times the optimal actions are correctly identified, which is consistent with the theoretical result of Theorem \ref{CSATheorem}. 

\subsection{When $|\mathcal{A}|$ is Polynomial in $d$.}
Here, we assume that we have prior knowledge of the rewards of arms, i.e., knowledge of the values of the items.
We assumed that we know each $v_s$ is in $[w_s, 1.1 \times w_s]$, and used this prior knowledge to generate the action class $\mathcal{A}$ in the following procedure. We first generated a vector $\boldsymbol{v}'$ whose $s$-th element $v'_s$ was uniformly sampled from $[w_s,  1.1 \times w_s]$, and then solved the knapsack problem with $v'_s$ and added the obtained solution $\boldsymbol{\pi}$ to $\mathcal{A}$. We repeated this 2000 times; therefore $|\mathcal{A}| \leq 2000$. In this experiment, the budget was $T = 50000$. For the Minimax-CombSAR algorithm, we set $\beta = 0.2, 0.4$. We ran forty experiments for each $d \in \{ 10, 15, \ldots, 100 \}$. We showed the result in Figure \ref{experiment2}. We can see that the Minimax-CombSAR algorithm outperforms the other two algorithms for almost every $d$. Also, the CSA algorithm outperforms the Peace algorithm for almost every $d$.

\section{Conclusion} \label{Conclusion_Section}
In this paper, we studied the fixed-budget R-CPE-MAB. We first introduced the CSA algorithm, which is the first algorithm that can identify the best action even when the size of the action class is exponentially large with respect to the number of arms. However, it still has an extra logarithmic term in the exponent. Then, we proposed an optimal algorithm named the Minimax-CombSAR algorithm, which, although it is applicable only when the action class is polynomial, matches a lower bound. We showed that both of the algorithms outperform the existing methods.

\section*{Acknowledgement}
We thank Dr. Kevin Jamieson for his very helpful advice and comments on existing studies.
SN was supported by JST SPRING, Grant Number JPMJSP2108.

\bibliography{main}
\bibliographystyle{apalike}

%%%%%%%%%%%%%%%%%%%%%%%%%%%%%%%%%%%%%%%%%%%%%%%%%%%%%%%%%%%%
\section*{Checklist}

% %%% BEGIN INSTRUCTIONS %%%
The checklist follows the references. For each question, choose your answer from the three possible options: Yes, No, Not Applicable.  You are encouraged to include a justification to your answer, either by referencing the appropriate section of your paper or providing a brief inline description (1-2 sentences). 
Please do not modify the questions.  Note that the Checklist section does not count towards the page limit. Not including the checklist in the first submission won't result in desk rejection, although in such case we will ask you to upload it during the author response period and include it in camera ready (if accepted).

\textbf{In your paper, please delete this instructions block and only keep the Checklist section heading above along with the questions/answers below.}
% %%% END INSTRUCTIONS %%%

 \begin{enumerate}

 \item For all models and algorithms presented, check if you include:
 \begin{enumerate}
   \item A clear description of the mathematical setting, assumptions, algorithm, and/or model. [Yes]
   \item An analysis of the properties and complexity (time, space, sample size) of any algorithm. [Yes]
   \item (Optional) Anonymized source code, with specification of all dependencies, including external libraries. [Yes]
 \end{enumerate}

 \item For any theoretical claim, check if you include:
 \begin{enumerate}
   \item Statements of the full set of assumptions of all theoretical results. [Yes]
   \item Complete proofs of all theoretical results. [Yes]
   \item Clear explanations of any assumptions. [Yes]     
 \end{enumerate}

 \item For all figures and tables that present empirical results, check if you include:
 \begin{enumerate}
   \item The code, data, and instructions needed to reproduce the main experimental results (either in the supplemental material or as a URL). [Yes]
   \item All the training details (e.g., data splits, hyperparameters, how they were chosen). [Yes]
         \item A clear definition of the specific measure or statistics and error bars (e.g., with respect to the random seed after running experiments multiple times). [Yes]
         \item A description of the computing infrastructure used. (e.g., type of GPUs, internal cluster, or cloud provider). [Yes]
 \end{enumerate}

 \item If you are using existing assets (e.g., code, data, models) or curating/releasing new assets, check if you include:
 \begin{enumerate}
   \item Citations of the creator If your work uses existing assets. [Yes]
   \item The license information of the assets, if applicable. [Not Applicable]
   \item New assets either in the supplemental material or as a URL, if applicable. [Yes]
   \item Information about consent from data providers/curators. [Yes]
   \item Discussion of sensible content if applicable, e.g., personally identifiable information or offensive content. [Not Applicable]
 \end{enumerate}

 \item If you used crowdsourcing or conducted research with human subjects, check if you include:
 \begin{enumerate}
   \item The full text of instructions given to participants and screenshots. [Not Applicable]
   \item Descriptions of potential participant risks, with links to Institutional Review Board (IRB) approvals if applicable. [Not Applicable]
   \item The estimated hourly wage paid to participants and the total amount spent on participant compensation. [Not Applicable]
 \end{enumerate}

 \end{enumerate}

\clearpage
\appendix

\thispagestyle{empty}

% For one-column format, uncomment the following:
\onecolumn
% \aistatstitle{Instructions for Paper Submissions to AISTATS 2024: \\
% Supplementary Materials}
% For two-column format, uncomment the following:
%\twocolumn[ \makesupplementtitle ]\section{Proof of Theorem \ref{LowerBoundTheorem}} \label{LowerBoundTheoremProof}
\section{Proof of Theorem \ref{thm:LowerBoundTheorem}} \label{LowerBoundTheoremProof}
Here, we prove Theorem \ref{thm:LowerBoundTheorem}. For the reader's convenience, we restate the Theorem.
\LowerBoundTheorem*
We follow a similar discussion to \citet{Carpentier2016}.
This is a lower bound that will hold in the much easier problem where the learner knows that the bandit setting she is facing is one of only $d$ given bandit settings. This lower bound ensures that even in this much simpler case, the learner will make a mistake. \par
Let $\mathcal{N}(\mu, 1)$ denote the Gaussian distribution of mean $\mu$ and unit variance. For any $k\in [d]$, we write $\nu_k:=\mathcal{N}\left( \mu_k, 1\right)$. Also, for any $k\in[d]$, we define $\nu'_k$ as follows:
\begin{equation}
    \nu'_k := \left\{
            \begin{array}{lll}
              \mathcal{N}(\mu_k + 2\Delta_{k}, 1) & (\mathrm{if} \ \pi^{*}_k < \pi_k^{(k)} )\\
              \mathcal{N}(\mu_k - 2\Delta_{k}, 1) & (\mathrm{if} \ \pi^{*}_k > \pi_k^{(k)} ) 
            \end{array}. \nonumber
              \right.
\end{equation}
Recall that $\Delta_{k} = \frac{ \boldsymbol{\mu}^{\top} \left(
\boldsymbol{\pi}^{*} - \boldsymbol{\pi}^{(k)} \right)}{\left| \pi_{k}^{*} - \pi^{(k)}_k \right|}$ and $\boldsymbol{\pi}^{(k)} = \argmin_{\boldsymbol{\pi} \neq \boldsymbol{\pi}^{*}} \frac{ \boldsymbol{\mu}^{\top} \left(
\boldsymbol{\pi}^{*} - \boldsymbol{\pi} \right)}{\left| \pi^{*}_{k} - \pi_k \right|}$. \par
For any $s\in [d]$, we define the product distributions $\mathcal{G}^{s}$ as $\nu_1^{s} \otimes \cdots \otimes \nu_d^{s}$, where for $1\leq k\leq d$, 
\begin{equation}
    \nu^{s}_{k} := \nu_k \mathbf{1}\left[ k \neq s \right] + \nu'_k \mathbf{1}\left[ k = s \right].
\end{equation}
The bandit problem associated with distribution $\mathcal{G}^{s}$, and that we call ``the bandit problem $s$'' is such that, for any $1\leq k \leq d$, arm $k$ has distribution $\nu^{s}_k$, i.e., all arms have distribution $\nu_k$ except for arm $s$ that has distribution $\nu'_k$. 
We denote by $\boldsymbol{\mu}^{(s)}$ the vector of expected rewards of bandit problem $s$, i.e., 
$\mu^{(s)}_k = \mu_k$ for $k \neq s$, 
$\mu^{(s)}_k = \mu_k + 2\Delta_{k}$ for $k = s$ and $\pi^{*}_k < \pi_k^{(k)}$,  
$\mu^{(s)}_k = \mu_k - 2\Delta_{k}$ for $k = s$ and $\pi^{*}_k >\pi_k^{(k)}$. 
For any $1 \leq s \leq d$, $\mathbb{P}_{s} := \mathbb{P}_{\left( \mathcal{G}^{s}\right)^{\otimes T}}$ for the probability distribution of the bandit problem $s$ according to all the samples that a strategy could possibly collect up to horizon $T$, i.e., according to the samples $(X_{k, u})_{1 \leq k \leq d, 1 \leq u \leq T} \sim \left( \mathcal{G}^{s}\right)^{\otimes T}$. Also, we define $\mathbb{P}_{0} := \mathbb{P}_{\left( \mathcal{G}^{0}\right)^{\otimes T}}$, where $\mathcal{G}^{0} := \nu_1 \otimes \cdots \otimes \nu_d$. 
\par
In the bandit problem $s$, $\boldsymbol{\pi}^{*}$ is no longer the best action since the difference of the expected reward between $\boldsymbol{\pi}^{*}$ and $\boldsymbol{\pi}^{(s)}$ is
\begin{eqnarray}
     &&\sum_{e \in [d] \setminus s} \mu_e (\pi^{*}_e - \pi_e^{(s)})  - 2 \Delta_s \times \left| \pi^{*}_s - \pi_s^{(s)} \right|  \nonumber \\
     &\leq& - \boldsymbol{\mu}^{\top}\left( \boldsymbol{\pi}^{*} - \boldsymbol{\pi}^{(s)} \right) < 0.
\end{eqnarray} 
We denote by $\boldsymbol{\pi}^{*, (s)}$ the best action in the bandit problem $s$.\par

Now, we prove Theorem \ref{LowerBoundTheorem}. \par

\textbf{Step 1: Definition of KL-divergence}\par
For two distributions $\nu$ and $\nu'$ defined on $\mathbb{R}$, and that are such that $\nu$ is absolutely continuous with respect to $\nu'$, we write 
\begin{equation}
    \mathrm{KL}(\nu, \nu') = \int_{\mathbb{R}} \log\left( \frac{d\nu(x)}{d\nu'(x)} \right) d\nu(x),
\end{equation}
for the Kullback Leibler divergence between distribution $\nu$ and $\nu'$. For any $k \in [d]$, let us write 
\begin{equation} \label{true_KL_explicit}
    \mathrm{KL}_{k} \triangleq \mathrm{KL}(\nu_k, \nu'_k) = \frac{1}{2} \left(\mu_a - \mu_b\right)^{2}.
\end{equation}
for the Kullback-Leibler divergence between two Gaussian distributions $\mathcal{N}(\mu_a, 1)$ and $\mathcal{N}(\mu_b, 1)$. \par
Let $1\leq t \leq T$. Also, let us fix $e \in [d]$, and think of bandit instance $e$. We define the quantity:
\begin{eqnarray}
    \widehat{\mathrm{KL}}_{s, t} 
    & = & \frac{1}{t} \sum_{u = 1}^{t} \log \left( \frac{d\nu_s}{d\nu'_s}\left(X_{s, u}\right) \right) \nonumber \\
    & = & \label{emp_KL_explicit} 
    \begin{cases}
    \frac{1}{t} \sum_{u = 1}^{t} - \Delta_{s} \left(X_{s, u} - \mu_s - 2 \Delta_{s} \right)              & \text{if$\pi^{*}_{s} > \pi^{(s)}_s$ and $e \neq s$,} \\
    \frac{1}{t} \sum_{u = 1}^{t} \Delta_{s} \left(X_{s, u} - \mu_s + 2 \Delta_{s} \right)        & \text{if $\pi^{*}_{s} < \pi^{(s)}_s$and $e \neq s$,} \\
    0 &\text{if $e = s$}
  \end{cases}
\end{eqnarray}
where $X_{s, u} \sim_{\mathrm{i.i.d}}\nu_{s}^{e}$ for any $u\leq t$.\par
Next, let us define the following event:
\begin{equation}
    \xi = \left\{ \forall 1\leq s \leq d, \forall 1 \leq t \leq T, \left| \widehat{\mathrm{KL}}_{s, t} \right| - \mathrm{KL}_{s} \leq  \sqrt{2 \Delta_{s}^2 \frac{\log (12Td)}{t} } \right\}.
\end{equation}
The following lemma shows a concentration bound for $\left| \widehat{\mathrm{KL}}_{s, t} \right|$.
\begin{lemma} \label{KLConcentrationLemma}
    It holds that 
    \begin{equation}
        {\Pr}_e\left[ \xi \right] \geq \frac{5}{6}
    \end{equation}
\end{lemma}
\begin{proof}
    When $e = s$, $\left| \widehat{\mathrm{KL}}_{s, t} \right| - \mathrm{KL}_{s} \leq  \sqrt{2 \Delta_{s}^2 \frac{\log (12Td)}{t} }$ holds for any $t$. Therefore, we think of $e \neq s$.
    Since $X_{s, t}$ is a $1$-sub-Gaussian random variable, we can see that $\widehat{\mathrm{KL}}_{s, t}$ is a $\frac{{\Delta_s}^2}{t}$-sub-Gaussian random variable from (\ref{emp_KL_explicit}). We can apply the Hoeffding's inequality to this quantity, and we have that with probability larger than $1 - \frac{1}{6dT}$, 
    \begin{equation}
       \left| \widehat{\mathrm{KL}}_{s, t} \right| - \mathrm{KL}_{s} \leq  \sqrt{ \frac{2\Delta_{s}^2\log (12Td)}{t} }
    \end{equation}
\end{proof}

\textbf{Step2: A change of measure} \par
Let $\mathcal{ALG}$ denote the active strategy of the learner, that returns some action $\boldsymbol{\pi}_{\mathrm{out}}$ after pulling arms $T$ times in total. Let $(T_s)_{1 \leq s \leq T}$ denote the number of samples collected by $\mathcal{ALG}$ on each arm. These quantities are stochastic but it holds that  $\sum_{s = 1}^{d} T_s = T$. For any $1 \leq s \leq d$, let us write 
\begin{equation}
t_s = \mathbb{E}_{0}[T_s].    
\end{equation} 
It holds also that $\sum_{s = 1}^{d} t_s = T$. \par
We recall the change of measure identity \citep{Audibert2010}, which states that for any measurable event $\mathcal{E}$ and for any $1 \leq s \leq d$,
\begin{equation}
    {\Pr}_{s} \left[ \mathcal{E} \right] = \mathbb{E}_{0} \left[ \mathbbm{1}\left[ \mathcal{E} \right] \exp\left( -T_s \widehat{\mathrm{KL}}_{s, T_s} \right) \right],
\end{equation}
where $\mathbbm{1}$ denotes the indicator function.\par
Then, we consider the following event:
\begin{equation}\label{ChangeofMeasureEq}
    \mathcal{E}_{s} = \left\{\boldsymbol{\pi}_{\mathrm{out}} = \boldsymbol{\pi}^{*}  \right\} \cap \left\{ \xi \right\} \cap \left\{ T_s \leq 6t_s \right\},
\end{equation}
i.e., the event where the algorithm outputs action $\boldsymbol{\pi}^{*}$ at the end, where $\xi$ holds, and where the number of times arm $s$ is pulled is smaller than $6t_s$. From (\ref{ChangeofMeasureEq}), we have
\begin{eqnarray}
    {\Pr}_{s}[\mathcal{E}_{s}] 
    & = & \mathbb{E}_{0} \left[ \mathbbm{1}\left[ \mathcal{E}_{s} \right] \exp\left( -T_s \widehat{\mathrm{KL}}_{s, T_s} \right) \right] \\
    & \geq & \mathbb{E}_{0} \left[ \mathbbm{1}\left[ \mathcal{E}_{s} \right] \exp\left( -T_s \mathrm{KL}_s - \sqrt{2T_s \Delta_{s}^2\log (12Td)} \right) \right] \\
    & \geq & \mathbb{E}_{0} \left[ \mathbbm{1}\left[ \mathcal{E}_{s} \right] \exp\left( -6t_s \mathrm{KL}_s - \sqrt{2T \Delta_{s}^2\log (12Td)} \right) \right] \\
    & \geq & \exp\left( -6t_s \mathrm{KL}_s - \sqrt{2T \Delta_{s}^2\log (12Td)} \right) {\Pr}_{0}\left[ \mathcal{E}_{s} \right], \label{LowerBoundProofEq4}
\end{eqnarray}
since on $\mathcal{E}_{s}$, we have that $\xi$ holds and that $T_s \leq 6t_s$, and $\mathbb{E}_{0}\left[\widehat{\mathrm{KL}}_{s, t}\right] = \mathrm{KL}_s$ for any $t \leq T$.\par
\textbf{Step 3: Lower bound on $\Pr_{0}\left[ \mathcal{E}_1 \right]$ for any reasonable algorithm}\\
Assume that the algorithm $\mathcal{ALG}$ is a $\delta$-correct algorithm. Then, we have ${\Pr}_{0}\left[ \boldsymbol{\pi}_{\mathrm{out}} \neq \boldsymbol{\pi}^{*} \right] \leq \delta$. From the Markov's inequality, we have 
\begin{equation}\label{MarkovEQ_on_NumPulls}
    {\Pr}_{0}[T_s \geq 6t_s] \leq \frac{\mathbb{E}_{0}T_s}{6t_s} = \frac{1}{6},
\end{equation}
since $\mathbb{E}_{0}\left[ T_s \right] = t_s$ for algorithm $\mathcal{ALG}$. \par
Combining Lemma \ref{KLConcentrationLemma}, (\ref{MarkovEQ_on_NumPulls}), and the fact that $\mathcal{ALG}$ is a $\delta$-correct algorithm, by the union bound, we have 
\begin{equation}
    {\Pr}_{0}\left[ \mathcal{E}_{s} \right] \geq 1 - \left( 1/6 + \delta + 1/6 \right) = \frac{2}{3} - \delta.
\end{equation}
This fact combined with (\ref{LowerBoundProofEq4}), (\ref{true_KL_explicit}), and the fact that ${\Pr}_{s}\left[\boldsymbol{\pi}_{\mathrm{out}} \neq \boldsymbol{\pi}^{*, (s)} \right] \geq \Pr_{s}\left[ \mathcal{E}_{s} \right]$, we have
\begin{eqnarray} \label{Eq7inLocatelli}
    {\Pr}_{s}\left[\boldsymbol{\pi}_{\mathrm{out}} \neq \boldsymbol{\pi}^{*, (s)} \right] 
    &\geq& \left( \frac{2}{3} - \delta \right) \exp\left( -6t_s \mathrm{KL}_s - \sqrt{2T \Delta_{s}^2\log (12Td)} \right)  \nonumber \\
    & = &  \left( \frac{2}{3} - \delta \right) \exp\left( -12t_s \left(\Delta_s\right)^{2} - \sqrt{2T \Delta_{s}^2 \log (12Td)} \right) \nonumber
\end{eqnarray}
since $\mathrm{KL}_{s} = 2\Delta_{s}^2$. \par
\textbf{Step 4: Conclusion.} Since $ \mathbf{H} = \sum\limits_{s = 1}^{d} \left( \frac{1}{\Delta_s}\right)^2$ and $\sum\limits_{s = 1}^{d}t_s = T$, there exists $e \in[d]$ such that
\begin{equation}
    t_{e} \leq \frac{T}{\mathbf{H} \Delta_{e}^2},
\end{equation}
as the contradiction yields an immediate contradiction. For this $e$, it holds by (\ref{Eq7inLocatelli}) that 
\begin{equation}
    {\Pr}_{e}\left[\boldsymbol{\pi}_{\mathrm{out}} \neq \boldsymbol{\pi}^{*, (e)} \right] 
    \geq  \left( \frac{2}{3} - \delta \right) \exp\left( - 12\frac{T}{\mathbf{H}} - 2\sqrt{T \Delta_{e}^2 \log (6Td)} \right).
\end{equation}
This concludes the proof.

\section{How to construct COracles} \label{COracleConstructionAppendix}
Here, we show how to construct COracles once $\boldsymbol{\mu}$ and $\boldsymbol{S}(t)$ are given for some specific combinatorial problems.
\subsection{COracle for the Knapsack Problem} \label{COracleConstructionAppendix_KnapsackProblem}
Here, we show that we can construct the \emph{COracle} for the knapsack problem by calling the \emph{offline oracle} once.
Let $\boldsymbol{S}^{1}(t)$ be the set that collects only the first component of each of all elements in $\boldsymbol{S}(t)$. Also, let $W_{\mathrm{done}} = \sum\limits_{(e, x) \in \boldsymbol{S}(t)} x w_e$. If $W - W_{\mathrm{done}} < 0$, the COracle outputs $\perp$. Otherwise, we call the \emph{offline oracle} and solve the following optimization problem:
\begin{equation*} \label{knapsackproblem_formulation}
\begin{array}{ll@{}ll}
\text{maximize}_{\boldsymbol{\xi}}  & \sum\limits_{s \in [d] \setminus \boldsymbol{S}^1(t)} \mu_{s}\xi_s  &\\ \\
\text{subject to}& \sum\limits_{s \in \boldsymbol{S}^1(t)} \xi_s w_s \leq W - W_{\mathrm{done}}. &
\end{array}
\end{equation*}
Then, we return $\mathrm{COracle}(\boldsymbol{\mu}, \boldsymbol{S}(t)) = \boldsymbol{S}(t) \cup \boldsymbol{\xi}$.
\subsection{COracle for the Optimal Transport Problem}\label{COracleConstructionAppendix_OptimalTransport}
Here, in Algorithm \ref{COracle_Optimal_Transport},  we show the \emph{COracle} for the optimal transport problem, which calls the \emph{offline oracle} once. First, let us define $\boldsymbol{S}^{1}(t)$ as the set that collects only the first element of each of all elements in $\boldsymbol{S}(t)$. Then, let us define $\boldsymbol{\pi}^{\mathrm{CO}}$ as the final output of the COracle. Also, let us define $\boldsymbol{s}'$ and $\boldsymbol{d}'$ as follows:
\begin{eqnarray}
    \forall i \in [m], j\in [n], x \in \texttt{POSSIBLE-PI}((i, j)) , \ {s}'_{i} = \left\{
            \begin{array}{ll}
            x & (\mathrm{if} \ ((i, j), x) \in \boldsymbol{S}(t), \\
            0 & (\mathrm{otherwise}) , \\
            \end{array} \nonumber
              \right.
\end{eqnarray}
and 
\begin{eqnarray}
    \forall i \in [m], j\in [n], x \in \texttt{POSSIBLE-PI}((i, j)) , \ {d}'_{j} = \left\{
            \begin{array}{ll}
            x & (\mathrm{if} \ ((i, j), x) \in \boldsymbol{S}(t), \\
            0 & (\mathrm{otherwise}) . \\
            \end{array} \nonumber
              \right.
\end{eqnarray}
Also, let us define $\boldsymbol{\pi}' \in \mathbb{R}^{m \times n}$ whose $(i, j)$-th element is defined as follows:
\begin{eqnarray}
    \forall i \in [m], j\in [n], x \in \texttt{POSSIBLE-PI}((i, j)) , \ \boldsymbol{\pi}'_{ij} = \left\{
            \begin{array}{ll}
            x & (\mathrm{if} \ ((i, j), x) \in \boldsymbol{S}(t), \\
            0 & (\mathrm{otherwise}) . \\
            \end{array} \nonumber
              \right.
\end{eqnarray}
Intuitively, for any $((i, j), x) \in \boldsymbol{S}(t)$, we have to send resources from supplier $i$ to demander $j$ for a value of $s'_i (= d'_j)$. This means that we have resources of $s_i - s'_i$ left to send from supplier $i$, and $d_j - d'_j$ left to send to demander $j$. Since we can not send a negative value of resources, $s_i - s'_i \geq 0$ and $d_j - d'_j \geq 0$ have to be satisfied. \par
From the above discussion, we can formally write $\boldsymbol{\pi}^{\mathrm{CO}}$ as follows:
\begin{eqnarray}
    \boldsymbol{\pi}^{\mathrm{CO}} =  \left\{
            \begin{array}{ll}
            \sum_{i, j}\pi'_{ij}\gamma_{ij} + \argmin_{\boldsymbol{\pi}'' \in \mathcal{G}\left( \boldsymbol{s} - \boldsymbol{s}', \boldsymbol{d} - \boldsymbol{d}' \right) } \sum_{i, j}\pi''_{ij}\gamma'_{ij} & (\text{ if there is no negative value in neither $\boldsymbol{s} - \boldsymbol{s}'$ nor $\boldsymbol{d} - \boldsymbol{d}'$}), \\
            \perp & (\mathrm{otherwise}) . \\
            \end{array} \nonumber
              \right. 
\end{eqnarray}
Here, $\boldsymbol{\gamma}'$ is defined as 
\begin{eqnarray}
    \forall i \in [m], j\in [n], x \in \texttt{POSSIBLE-PI}((i, j)) , \ {\gamma}'_{ij} = \left\{
            \begin{array}{ll}
            - \infty & (\mathrm{if} \ ((i, j), x) \in \boldsymbol{S}(t), \\
            \gamma_{ij} & (\mathrm{otherwise}) , \\
            \end{array} \nonumber
              \right.
\end{eqnarray}
to not let $\boldsymbol{\pi}^{\mathrm{CO}}$ send materials more than the amount determined by $\boldsymbol{S}(t)$.

\begin{algorithm}
    \caption{COracle for the Optimal Transport Problem} \label{COracle_Optimal_Transport}
    \begin{algorithmic}[1]
        \renewcommand{\algorithmicrequire}{\textbf{Input:}}
         \renewcommand{\algorithmicensure}{\textbf{Parameter:}}
         \REQUIRE Cost matrix: $\boldsymbol{\gamma}$, supply vector $\boldsymbol{s} \in \mathbb{R}^{m}$, demand vector  $\boldsymbol{d} \in \mathbb{R}^{n}$, $\boldsymbol{S}(t)$
         \STATE $\boldsymbol{\pi}' \leftarrow \boldsymbol{0}^{m \times n}$
         \FOR{$((i, j), x)$ in $\boldsymbol{S}(t)$} 
            \STATE $\pi'_{ij} \leftarrow x$ \label{final_output_determine}
            \STATE $s_i \leftarrow s_i - x$ \label{supply_decrease_line}
            \STATE $d_j \leftarrow d_j - x$ \label{demand_decrease_line}
            \STATE $\gamma_{ij} \leftarrow - \infty $
            \IF{$s_i < 0$ or $d_j < 0$} \label{Conditional_Branch}
            \STATE Return $\perp$
            \ENDIF
         \ENDFOR
         \STATE Compute $\boldsymbol{\pi}'' =  \argmin_{\boldsymbol{\pi}\in \mathcal{G}(\boldsymbol{s}, \boldsymbol{d})} \sum_{i, j}\pi_{ij}\gamma_{ij}$ using the offline oracle \label{remain_compute_line}
         \STATE Return $\boldsymbol{\pi}^{\mathrm{CO}} = \boldsymbol{\pi}' + \boldsymbol{\pi}''$
    \end{algorithmic}
\end{algorithm}
% Let us intuitively explain Algorithm \ref{COracle_Optimal_Transport}. \par
% Let us define $\boldsymbol{S}^{1}(t)$ as the set that collects only the first element of each of all elements in $\boldsymbol{S}(t)$. $\boldsymbol{\pi}^{\mathrm{CO}}$ is the final output of the COracle.  For any $((i, j), x) \in \boldsymbol{S}(t)$, $\boldsymbol{\pi}'$ maintains the $(i, j)$-th element of $\boldsymbol{\pi}^{\mathrm{CO}}$, which is $x$ (line \ref{final_output_determine}). 
% For any $(i, j) \in [m] \times [n] \setminus \boldsymbol{S}^{1}(t)$, $\boldsymbol{\pi}''$ determines the $(i, j)$-th element of $\boldsymbol{\pi}^{\mathrm{CO}}$. \par
% For any $((i, j), x) \in \boldsymbol{S}(t)$, we have to send resources from supplier $i$ to demander $j$ for a value of $x$. Therefore, we have resources of $s_i - x$ left to send from supplier $i$, and $d_j - x$ left to send to demander $j$ (line \ref{supply_decrease_line}--\ref{demand_decrease_line}). If $s_i < 0$ 

\subsection{COracle for a General Case When $K (=|\mathcal{A}|) = \mathrm{poly}(d)$} \label{COracleConstructionAppendix_GeneralCase}
Here, we show the time complexity of the COracle when the size of the action class is polynomial in $d$. If $\mathcal{A}_{S}$ is an empty set, the COracle returns $\perp$. Otherwise, we return $\boldsymbol{\pi}^{\mathrm{return}} = \argmax_{\boldsymbol{\pi} \in \mathcal{A}_{S}} \boldsymbol{\mu}^{\top} \boldsymbol{\pi}$. The time complexity to construct $\mathcal{A}_{S}$ is $\mathcal{O}(d K)$. This is because, for each action $\boldsymbol{\pi} \in \mathcal{A}$, we check every dimension to see if $\boldsymbol{\pi} \in \mathcal{A}_{S}$. Then, the time complexity of finding $\boldsymbol{\pi}^{\mathrm{return}}$ from $\mathcal{A}_{S}$ is $\mathcal{O}(K \log K)$. Therefore the time complexity of the COracle is $\mathcal{O}\left( dK + K \log K \right)$.

\section{Proof of Theorem \ref{CSATheorem}}
Here, we prove Theorem \ref{CSATheorem}. We first introduce some notions that are useful to prove it. Then, we show some preparatory lemmas that are needed to prove Theorem \ref{CSATheorem}. 
\subsection{Preparatories}
Let us introduce some notions that are useful to prove the theorem.
\subsubsection{Arm-Value Pair}
We define an \emph{arm-value pair set} $\boldsymbol{M}(\boldsymbol{\pi}) \subset [d] \times \mathbb{R}$ of $\boldsymbol{\pi}$ as follows:
\begin{eqnarray}
    \boldsymbol{M}(\boldsymbol{\pi}) = \{ (e, \pi_e) \ | \ \forall e\in [d] \}. \nonumber
\end{eqnarray}
Also, we define the \emph{arm-value pair family} $\mathcal{M}$ as follows:
\begin{eqnarray}
    \mathcal{M} = \{ \boldsymbol{M}(\boldsymbol{\pi}) \ | \ \forall \boldsymbol{\pi} \in \mathcal{A} \}. \nonumber
\end{eqnarray}
For any arm-value pair set $\boldsymbol{M} \in \mathcal{M}$, we denote by $M_e$ the second component of the ordered pair whose first component is $e$, i.e., $M_e = x$ for any $(e, x) \in \boldsymbol{M}$. Also, we call $M_e$ the $e$-th element of $\boldsymbol{M}$. \par
For any two different arm-value pair set $\boldsymbol{M}, \boldsymbol{M}' \in \mathcal{M}$, we define an operator $\boxminus$ such that the $e$-th element of $\boldsymbol{M}'\boxminus \boldsymbol{M} \subset [d] \times \mathbb{R}$, $(\boldsymbol{M}'\boxminus \boldsymbol{M})_{e}$, is defined as follows:
\begin{eqnarray}
    (\boldsymbol{M}' \boxminus \boldsymbol{M})_e = \left\{
            \begin{array}{lll}
              M'_e - M_e & (\mathrm{if} \ M'_e > M_e )\\
              0 & (\mathrm{if} \ M'_e \leq M_e )
            \end{array}. \nonumber
              \right.
\end{eqnarray}

\subsubsection{Exchange Class}
We define an \emph{exchange set} $b$ as an ordered pair of sets $b = (\boldsymbol{b}^+, \boldsymbol{b}^-)$, where $\boldsymbol{b}^{+}, \boldsymbol{b}^{-} \subset [d] \times \mathbb{R}\setminus \{0\}$. We say $\boldsymbol{b}^{+}$ (or $\boldsymbol{b}^{-}$) has as \emph{$e$-changer} if it has an element whose first component is $e$. For any $\boldsymbol{b}^{+}$ (or $\boldsymbol{b}^{-}$), we denote by $b^{+}_e$ the second component of the ordered pair whose first component is $e$, i.e., $b^{+}_e = x$ for any $(e, x) \in \boldsymbol{b}^{+}$.
Also, for any $b = (\boldsymbol{b}^{+}, \boldsymbol{b}^{-})$, we do not let both $\boldsymbol{b}^{+}$ and $\boldsymbol{b}^{-}$ have $e$-changers. \par
Next, for any arm-value pair set $\boldsymbol{M} \in \mathcal{M}$, exchange set $b = (\boldsymbol{b}^{+}, \boldsymbol{b}^{-})$, and $e\in[d]$, we define operator $\oplus$ such that the $e$-the element of $\boldsymbol{M} \oplus b \subset [d] \times \mathbb{R}$, $(\boldsymbol{M} \oplus b)_{e}$, is defined as follows:
\begin{equation}
        (\boldsymbol{M} \oplus b)_e = \left\{
            \begin{array}{lll}
              M_e + b^{+}_e & (\mathrm{if} \ \boldsymbol{b}^{+} \ \mathrm{has \ an \  \mathit{e-}changer})\\
              M_e - b^{-}_e & (\mathrm{if} \ \boldsymbol{b}^{-} \ \mathrm{has \ an \  \mathit{e-}changer})\\
              M_e  & (\mathrm{if \ neither } \ \boldsymbol{b}_{+} \ \mathrm{nor} \ \boldsymbol{b}^{-} \ \mathrm{has \ an \  \mathit{e-}changer})\\
            \end{array}.\label{oplus_Definition} 
              \right.
\end{equation}
Similarly, for any arm-value pair set $\boldsymbol{M} \in \mathcal{M}$, exchange set $b = (\boldsymbol{b}^{+}, \boldsymbol{b}^{-})$, and $e\in[d]$, we define operator $\ominus$ such that the $e$-the element of $\boldsymbol{M} \ominus b \subset [d] \times \mathbb{R}$, $(\boldsymbol{M} \ominus b)_e$, is defined as follows:
\begin{equation}
        (\boldsymbol{M} \ominus b)_e = \left\{
            \begin{array}{lll}
              M_e - b^{+}_e & (\mathrm{if} \ \boldsymbol{b}^{+} \ \mathrm{has \ an \  \mathit{e-}changer})\\
              M_e + b^{-}_e & (\mathrm{if} \ \boldsymbol{b}^{-} \ \mathrm{has \ an \  \mathit{e-}changer})\\
              M_e  & (\mathrm{if \ neither } \ \boldsymbol{b}^{+} \ \mathrm{nor} \ \boldsymbol{b}^{-} \ \mathrm{has \ an \  \mathit{e-}changer})\\
            \end{array}. \label{ominus_Definition}
              \right.
\end{equation}
We call a collection of exchange sets $\mathcal{B}$ an \emph{exchange class} for $\mathcal{M}$ if $\mathcal{B}$ satisfies the following property. For any $\boldsymbol{M}, \boldsymbol{M}' \in \mathcal{M}$ and $e\in[d]$, where $(\boldsymbol{M}\boxminus \boldsymbol{M}')_e > 0$, there exists an exchange set $b = (\boldsymbol{b}^{+}, \boldsymbol{b}^{-}) \in \mathcal{B}$ that satisfies the following five constraints: 
\textbf{(a)} $b^{-}_e = (\boldsymbol{M} \boxminus \boldsymbol{M}')_e$ \textbf{(b)} $\boldsymbol{b}^{+}\subseteq \boldsymbol{M}'\boxminus \boldsymbol{M}$, \textbf{(c)} $\boldsymbol{b}^{-}\subseteq \boldsymbol{M} \boxminus \boldsymbol{M}'$, \textbf{(d)}$(\boldsymbol{M} \oplus b) \in \mathcal{M}$ and \textbf{(e)} $(\boldsymbol{M}' \ominus b) \in \mathcal{M}$.\par
For any $\boldsymbol{a} \subset [d] \times \mathbb{R}$, let $a_e$ denote the second component of the ordered pair whose first component is $e$, i.e., $a_e = x$ for any $(e, x) \in \boldsymbol{a}$. Then, let $\boldsymbol{\chi}(\boldsymbol{a}) \in \mathbb{R}^{d}$ denote the vector defined as follows:
\begin{eqnarray}
    \chi_e(\boldsymbol{a}) = \left\{
            \begin{array}{lll}
              a_e & (\text{if $\boldsymbol{a}$ has an element whose first component is $e$ })\\
              0 & (\text{otherwise})
            \end{array}. \nonumber
    \right.
\end{eqnarray}
Also, for any exchange set $b$, we define $\boldsymbol{\chi}(b) = \boldsymbol{\chi}(\boldsymbol{b}^{+}) - \boldsymbol{\chi}(\boldsymbol{b}^{-})$. 

\subsection{Preparatory Lemmas}
Here, let us introduce some lemmas that are useful to prove the theorem. Below, we define $\boldsymbol{M}^{*} = \boldsymbol{M}(\boldsymbol{\pi}^{*})$.
\begin{lemma}[Interpolation Lemma] \label{InterpolationLemma}
    Let $\mathcal{B}$ be an exchange class of $\mathcal{M}$, and $\boldsymbol{M}, \boldsymbol{M}'$ be two different members of $\mathcal{M}$. Then, for any $e \in \{ s\in [d] \ | \ M_s \neq M'_s \}$, there exists an exchange set $b = (\boldsymbol{b}^{+}, \boldsymbol{b}^{-}) \in \mathcal{B}$, which satisfies the following five constraints: 
    \textbf{(a)} $b^{-}_e = (\boldsymbol{M}\boxminus \boldsymbol{M}')_e  $ or $  b^{+}_e = (\boldsymbol{M}'\boxminus \boldsymbol{M})_e $, 
    \textbf{(b)} $\boldsymbol{b}^{-} \subseteq (\boldsymbol{M}\boxminus \boldsymbol{M}')$, 
    \textbf{(c)} $\boldsymbol{b}^{+} \subseteq \boldsymbol{M}'\boxminus \boldsymbol{M}$, 
    \textbf{(d)} $(\boldsymbol{M} \oplus b) \in \mathcal{M}$ and \textbf{(e)}$(\boldsymbol{M}'\ominus b) \in \mathcal{M}$. 
    Moreover, if $\boldsymbol{M}' = \boldsymbol{M}^{*}$, then we have 
    \begin{eqnarray}
    \frac{\langle \boldsymbol{\mu}, \boldsymbol{\chi}(b) \rangle}{\chi_e(b)} \geq \Delta_{e}.        
    \end{eqnarray}
\end{lemma}
\begin{proof}
    We decompose our proof into two cases. \par
    \textbf{Case (1)}: $(\boldsymbol{M} \boxminus \boldsymbol{M}')_e > 0$ \par
    By the definition of the exchange class, we know that there exists $b = (\boldsymbol{b}^{+}, \boldsymbol{b}^{-}) \in \mathcal{B}$ which satisfies that there is an $e\in[d]$ where 
    \textbf{(a)} $b^{-}_e = (\boldsymbol{M}\boxminus \boldsymbol{M}')_e  $, 
    \textbf{(b)} $\boldsymbol{b}^{-} \subseteq (\boldsymbol{M}\boxminus \boldsymbol{M}')$, 
    \textbf{(c)} $\boldsymbol{b}^{+} \subseteq \boldsymbol{M}'\boxminus \boldsymbol{M}$, 
    \textbf{(d)} $(\boldsymbol{M} \oplus b) \in \mathcal{M}$ and \textbf{(e)}$(\boldsymbol{M}'\ominus b) \in \mathcal{M}$. Therefore the five constraints are satisfied.\par
    \textbf{Case (2)}: $(\boldsymbol{M}' \boxminus \boldsymbol{M})_e > 0$ \par
    Using the definition of the exchange class, we see that there exists $b = (\boldsymbol{c}^{+}, \boldsymbol{c}^{-}) \in \mathcal{B}$ such that 
    \textbf{(a)} $c^{-}_e = (\boldsymbol{M}'\boxminus \boldsymbol{M})_e$, 
    \textbf{(b)} $\boldsymbol{c}^{-}\subseteq (\boldsymbol{M}' \boxminus \boldsymbol{M})$, 
    \textbf{(c)} $ \boldsymbol{c}^{+}\subseteq (\boldsymbol{M} \boxminus \boldsymbol{M}')$, 
    \textbf{(d)} $(\boldsymbol{M}\oplus b)\in \mathcal{M}$, and 
    \textbf{(e)} $(\boldsymbol{M}' \ominus b)\in \mathcal{M}$. 
    We construct $b = (\boldsymbol{b}^{+}, \boldsymbol{b}^{-})$ by setting $\boldsymbol{b}^{+} = \boldsymbol{c}^{-}$ and $\boldsymbol{b}^{-} = \boldsymbol{c}^{+}$. Notice that, by the construction of $b$, we have $\boldsymbol{M} \oplus b = \boldsymbol{M}\ominus c$ and $\boldsymbol{M}'\ominus b = \boldsymbol{M}'\oplus c$. Therefore, it is clear that $b$ satisfies the five constraints of the lemma. \par
    Next, let us think when $\boldsymbol{M}' = \boldsymbol{M}^{*}$ for both cases. 
    Let us consider 
    \begin{eqnarray}
        \boldsymbol{\pi}^1 = \argmin_{\boldsymbol{\pi}\in \mathcal{M} \setminus \boldsymbol{\pi}^{*}} \frac{\langle \boldsymbol{\mu}, \boldsymbol{\pi}^{*} - \boldsymbol{\pi}\rangle}{\left| \pi^{*}_e - \pi_e \right|}.
    \end{eqnarray} 
    Note that, by the definition of the \emph{G-Gap}, we have $\frac{\langle \boldsymbol{\mu}, \boldsymbol{\pi}^{*} - \boldsymbol{\pi}^{1}\rangle}{\left|\pi^{*}_e - \pi^1_e\right|} = \Delta_{e}$. 
    We define $\boldsymbol{\pi}^{0}$ such that $\boldsymbol{M}(\boldsymbol{\pi}^{0}) = \boldsymbol{M}(\boldsymbol{\pi}^{*}) \ominus b$. 
    Note that we already have $\boldsymbol{M}(\boldsymbol{\pi}^{0}) \in \mathcal{M}$.
    We can see that 
    \begin{eqnarray}
     \frac{\langle \boldsymbol{\mu}, \boldsymbol{\chi}(b) \rangle}{\left|\chi_e(b) \right|} = \frac{\langle \boldsymbol{\mu}, \boldsymbol{\pi}^{*} - \boldsymbol{\pi}^{0} \rangle}{\left|\pi^{*}_e - \pi^{0}_{e}\right|} \geq  \Delta_{e},
    \end{eqnarray}
    where the inequality follows from the definition of \emph{G-Gap}.
\end{proof}

Next, we establish the confidence bounds used for the analysis of the CSA algorithm.
\begin{lemma}\label{Lemma16}
    Given a phase $t\in [d]$, we define random events $\tau_t$ as follows:
    \begin{eqnarray}
        \tau(t) = \left\{ \forall s \in [d] \setminus F(t), | \hat{\mu}_s(t) - \mu_s | < \frac{\Delta_{(d - t + 1)}}{(2 + L^2)U_{\mathcal{A}}} \right\}, \label{tau_definition}
    \end{eqnarray}
    where $L = \max_{ e\in[d], \boldsymbol{\pi}^{1}, \boldsymbol{\pi}^{2}, \boldsymbol{\pi}^{3} \in \mathcal{A}} \frac{\left| \pi^{1}_e - \pi^{2}_e \right|}{\left| \pi^{1}_{e} - \pi^{3}_{e} \right|}$.
    Then, we have
    \begin{eqnarray}
        \Pr \left[ \bigcap_{t = 1}^{d} \tau(t) \right] \geq 1 - d^2 \exp\left( - \frac{T - d}{2\left(2 + L^2\right)^{2}R^2 \tilde{\log}(d) U^2_{\mathcal{M}} \mathbf{H}_2 } \right)
    \end{eqnarray}
\end{lemma}
\begin{proof}
    Fix some $t\in[d]$ and active arm $s\in[d] \setminus F(t)$ of phase $t$. Note that arm $s$ has been pulled for $\tilde{T}(t)$ times during the first $t$ phases. Therefore, by Hoeffding's inequality, we have
    \begin{eqnarray}
        \Pr \left[ \left| \hat{\mu}_s(t) - \mu_s \right| \geq \frac{\Delta_{(d - t + 1)}}{(2 + L^2)U_{\mathcal{A}}} \right] \leq 2 \exp \left( - \frac{\tilde{T}(t) \Delta^{2}_{(d - t + 1)}}{2\left(2 + L^2\right)^{2} R^2 U^2_{\mathcal{A}}} \right). \label{Eq88}
    \end{eqnarray}
    By plugging the definition of $\tilde{T}(t)$, the quantity $\Tilde{T}(t) \Delta^2_{(d - t + 1)}$ on the right hand side of (\ref{Eq88}) can be further bounded by
    \begin{eqnarray}
        \Tilde{T}(t) \Delta^{2}_{(d - t + 1)} 
        & \geq & \frac{T - d}{ \Tilde{\log}(d) (d - t + 1) } \Delta^2_{(d - t + 1)} \nonumber \\
        & \geq & \frac{T - d}{ \Tilde{\log}(d) \mathbf{H}_2 }, \nonumber
    \end{eqnarray}
    where the last inequality follows from the definition of $\mathbf{H}_2 = \max_{s\in[d]} \frac{s}{\Delta^2_{(s)}}$. By plugging the last equality into (\ref{Eq88}), we have
    \begin{eqnarray}
        \Pr \left[ \left| \hat{\mu}_s(t) - \mu_s \right| \geq \frac{\Delta_{(d - t + 1)}}{(2 + L^2)U_{\mathcal{A}}} \right] \leq 2 \exp \left( - \frac{T - d}{2\left(2 + L^2\right)^{2} R^2 \Tilde{\log}(d) U^2_{\mathcal{A}} \mathbf{H}_{2}} \right). \label{Eq89}
    \end{eqnarray}
    Using (\ref{Eq89}) and a union bound for all $t \in [d]$ and all $s\in [d] \setminus F(t)$, we have
    \begin{eqnarray}
        \Pr \left[ \bigcap_{t = 1}^{d} \tau(t) \right] 
        & \geq &  1 - 2 \sum_{t = 1}^{d} (d - t + 1)  \exp\left( - \frac{T - d}{2\left(2 + L^2\right)^{2} R^2 \tilde{\log}(d) U^2_{\mathcal{A}} \mathbf{H}_2 } \right) \nonumber \\
        & \geq &  1 - d^2 \exp\left( - \frac{T - d}{2\left(2 + L^2\right)^{2} R^2 \tilde{\log}(d) U^2_{\mathcal{M}} \mathbf{H}_2 } \right) \nonumber
    \end{eqnarray}
\end{proof}
The following lemma builds the confidence bound of inner products.
\begin{lemma}
    Fix a phase $t \in [d]$. Suppose that random event $\tau(t)$ occurs. For any vector $\boldsymbol{a} \in \mathbb{R}^{d}$, we have
    \begin{equation}
        |\langle \hat{\boldsymbol{\mu}}(t) , \boldsymbol{a} \rangle - \langle \boldsymbol{\mu}, \boldsymbol{a}\rangle| < \frac{\Delta_{(d - t + 1)}}{(2 + L^2)U_{\mathcal{A}}} \| \boldsymbol{a} \|_1.
    \end{equation}
\end{lemma}
\begin{proof}
    Suppose that $\tau_t$ occurs. We have 
    \begin{eqnarray}
        |\langle \hat{\boldsymbol{\mu}}(t) , \boldsymbol{a} \rangle - \langle \boldsymbol{\mu}, \boldsymbol{a} \rangle | 
        & = & | \langle \hat{\boldsymbol{\mu}}(t) - \boldsymbol{\mu}, \boldsymbol{a} \rangle | \nonumber \\
        & = & \left| \sum_{s = 1}^{d} (\hat{\mu}_s(t) - \mu_s) a_s \right| \nonumber \\
        & \leq & \sum_{s = 1}^{d} |\hat{\mu}_s(t) - \mu_s| |a_s| \nonumber \\
        & < & \frac{\Delta_{(d - t + 1)}}{(2 + L^2)U_{\mathcal{A}}} \sum_{i = 1}^{d} |a_s| \label{Eq91} \\
        & = & \frac{\Delta_{(d - t + 1)}}{(2 + L^2)U_{\mathcal{A}}} \|\boldsymbol{a} \|_1, \nonumber
    \end{eqnarray}
    where (\ref{Eq91}) follows from the definition of $\tau_t$ in (\ref{tau_definition}).
\end{proof}

\subsection{Main Lemmas}
Let $\boldsymbol{R}(t) = \{ (e, x) \ | \ \forall e \in [d], \forall x \in \texttt{POSSIBLE-PI(e)}, (e, x) \notin \boldsymbol{S}(t)\}$. We begin with a technical lemma that characterizes several useful lemma properties of $\boldsymbol{S}(t)$ and $\boldsymbol{R}(t)$.
\begin{lemma} \label{Lemma18}
    Fix a phase $t\in[d]$. Suppose that $\boldsymbol{S}(t)\subseteq \boldsymbol{M}^{*}$ and $\boldsymbol{R}(t) \cap \boldsymbol{M}^{*} = \emptyset$. Let $\boldsymbol{M}$ be a set such that $\boldsymbol{S}(t) \subseteq \boldsymbol{M}$ and $\boldsymbol{R}(t) \cap \boldsymbol{M} = \emptyset$. Let $\boldsymbol{a}$ and $\boldsymbol{b}$ be two sets satisfying that $\boldsymbol{a}\subseteq \boldsymbol{M} \setminus \boldsymbol{M}^{*}$, $\boldsymbol{b} \subseteq \boldsymbol{M}^{*} \setminus \boldsymbol{M}$ and $\boldsymbol{a} \cap \boldsymbol{b} = \emptyset$. Then, we have 
    \begin{center}
        $\boldsymbol{S}(t) \subseteq \left(\boldsymbol{M}\setminus \boldsymbol{a} \cup \boldsymbol{b}\right)$ \ \ and \ \ $\boldsymbol{R}(t)  \cap (\boldsymbol{M} \setminus \boldsymbol{a} \cup \boldsymbol{b}) = \emptyset$
    \end{center}
\end{lemma}
\begin{proof}
    We first prove the first part as follows:
    \begin{eqnarray}
        \boldsymbol{S}(t) \cap \left(\boldsymbol{M} \setminus \boldsymbol{a} \cup \boldsymbol{b} \right) 
        & = & \left(\boldsymbol{S}(t) \cap \left(\boldsymbol{M} \setminus \boldsymbol{a} \right)\right) \cup \left(\boldsymbol{S}(t) \cap \boldsymbol{b}\right) \nonumber \\ 
        & = & \boldsymbol{S}(t) \cap \left( \boldsymbol{M} \setminus \boldsymbol{a} \right)  \label{Eq92} \\ 
        & = & \left( \boldsymbol{S}(t) \cap \boldsymbol{M} \right) \setminus \boldsymbol{a} \nonumber \\ 
        & = & \boldsymbol{S}(t) \setminus \boldsymbol{a} \label{Eq93} \\
        & = & \boldsymbol{S}(t), \label{Eq94}
    \end{eqnarray}
    where (\ref{Eq92}) holds since we have $\boldsymbol{S}(t) \cap \boldsymbol{b} \subseteq \boldsymbol{S}(t) \cap (\boldsymbol{M}^{*} \setminus \boldsymbol{M}) \subseteq \boldsymbol{M} \cap (\boldsymbol{M}^{*} \setminus \boldsymbol{M}) = \emptyset$; (\ref{Eq93}) follows from $\boldsymbol{S}(t) \subseteq \boldsymbol{M}$; and (\ref{Eq94}) follows from $\boldsymbol{a} \subseteq \boldsymbol{M} \setminus \boldsymbol{M}^{*}$ and $\boldsymbol{S}(t) \subseteq \boldsymbol{M}^{*}$ which implies that $\boldsymbol{a} \cap \boldsymbol{S}(t) = \emptyset$. Notice that $\boldsymbol{S}(t) \subseteq (\boldsymbol{M} \setminus \boldsymbol{a} \cup \boldsymbol{b})$. \par
    Then, we proceed to prove the second part in the following
    \begin{eqnarray}
        \boldsymbol{R}(t) \cap (\boldsymbol{M}\setminus \boldsymbol{a}\cup \boldsymbol{b}) 
        & = & (\boldsymbol{R}(t) \cap (\boldsymbol{M} \setminus \boldsymbol{a})) \cup (\boldsymbol{R}(t)\cap \boldsymbol{b}) \nonumber \\
        & = & \boldsymbol{R}(t) \cap (\boldsymbol{M} \setminus \boldsymbol{a}) \label{Eq95} \\
        & = & (\boldsymbol{R}(t) \cap \boldsymbol{M}) \setminus \boldsymbol{a} \nonumber \\
        & = & \emptyset \setminus \boldsymbol{a} = \emptyset, \label{Eq96}
    \end{eqnarray}
    where (\ref{Eq95}) follows from the fact that $\boldsymbol{R}(t) \cap \boldsymbol{b} \subseteq \boldsymbol{R}(t) \cap (\boldsymbol{M}^{*} \setminus M) \subseteq R(t) \cap \boldsymbol{M}^{*} = \emptyset$; and (\ref{Eq96}) follows from the fact that $\boldsymbol{R}(t) \cap \boldsymbol{M} = \emptyset$. \par
\end{proof}

Let $\hat{\boldsymbol{M}}(t) = \boldsymbol{M}(\hat{\boldsymbol{\pi}}(t))$. 
The next lemma provides an important insight into the correctness of the CSA algorithm. Informally speaking, suppose that the algorithm does not make an error before phase $t$. Then, we show that, if arm $e$ has a gap $\Delta_e$ larger than the ``reference gap'' $\Delta_{(d - t + 1)}$ of phase $t$, then arm $e$ must be correctly clarified by $\hat{\boldsymbol{M}}(t)$, i.e., $M_e(t) = M^{*}_e$. \par
\begin{lemma} \label{Lemma19}
    Fix any phase $t>0$. Suppose that event $\tau_t$ occurs. Also, assume that $\boldsymbol{S}(t) \subseteq \boldsymbol{M}^{*}$ and $\boldsymbol{R}(t)\cap \boldsymbol{S}(t) = \emptyset$. Let $e \in [d]\setminus \boldsymbol{F}(t)$ be an active arm. Suppose that $\Delta_{(d - t + 1)} \leq \Delta_e$. Then, we have $(e, \pi^{*}_e) \in \boldsymbol{M}^{*}\cap \boldsymbol{S}(t)$.
\end{lemma}
\begin{proof}
    Suppose that $(e, \pi^{*}_e)\notin ( \boldsymbol{M}^{*}\cap \hat{\boldsymbol{M}}(t))$. 
    This is equivalent to the following
    \begin{equation}\label{Eq99}
        (e, \pi^{*}_e)\in (\boldsymbol{M}^{*}\cap \lnot \hat{\boldsymbol{M}}(t)) \cup (\lnot \boldsymbol{M}^{*} \cap \hat{\boldsymbol{M}}(t)) 
    \end{equation}
    (\ref{Eq99}) can be further rewritten as 
    \begin{equation}
        (e, \pi^{*}_e)\in (\boldsymbol{M}^{*} \setminus \hat{\boldsymbol{M}}(t)) \cup (\hat{\boldsymbol{M}}(t) \setminus \boldsymbol{M}^{*}).
    \end{equation}
    From this assumption, it is easy to see that $\hat{\boldsymbol{M}}(t) \neq \boldsymbol{M}^{*}$. Therefore, we can apply Lemma \ref{InterpolationLemma}. 
    We know that there exists $b = (\boldsymbol{b}^{+}, \boldsymbol{b}^{-}) \in \mathcal{B}$ such that 
    \textbf{(a)} ${b}^{-}_e = (\hat{\boldsymbol{M}}(t) \boxminus \boldsymbol{M}^{*})_e  $ or $ {b}^{+}_e = (\boldsymbol{M}^{*}\boxminus \hat{\boldsymbol{M}}(t))_e$, 
    \textbf{(b)} $\boldsymbol{b}^{-}\subseteq (\hat{\boldsymbol{M}}(t) \boxminus \boldsymbol{M}^{*})$, 
    \textbf{(c)} $\boldsymbol{b}^{+} \subseteq \boldsymbol{M}^{*} \boxminus \hat{\boldsymbol{M}}(t)$, 
    \textbf{(d)} $(\hat{\boldsymbol{M}}(t)\oplus b) \in \mathcal{M}$, 
    \textbf{(e)} $(\boldsymbol{M}^{*} \ominus b) \in \mathcal{M}$,
    and $\frac{\langle \boldsymbol{\mu}, \boldsymbol{\chi}(b) \rangle}{\chi_e(b)} \geq  \Delta_e > 0$. \par
    Using Lemma \ref{Lemma18}, we see that $(\hat{\boldsymbol{M}}(t) \oplus b)\cap \boldsymbol{R}(t) = \emptyset$, $\boldsymbol{S}(t) \subseteq (\hat{\boldsymbol{M}}(t) \oplus b)$, and $(\boldsymbol{b}^{+} \cup \boldsymbol{b}^{-}) \cap (\boldsymbol{S}(t) \cup \boldsymbol{R}(t) ) = \emptyset$. \par 
    Recall the definition $\hat{\boldsymbol{M}}(t) \in \argmax_{\boldsymbol{M} \in \mathcal{M}, \boldsymbol{S}(t) \subseteq \boldsymbol{M}, \boldsymbol{R}(t) \cap \boldsymbol{M} = \emptyset} \langle \hat{\boldsymbol{\mu}}(t), \boldsymbol{\pi}(\boldsymbol{M}) \rangle$ and also recall that $\hat{\boldsymbol{M}}(t) \oplus b \in \mathcal{M}$. 
    Therefore, we obtain that
    \begin{equation}
        \frac{\langle \hat{\boldsymbol{\mu}}(t), \boldsymbol{\pi}(\hat{\boldsymbol{M}}(t)) \rangle}{\chi_e(b)} \geq \frac{\langle \hat{\boldsymbol{\mu}}, \boldsymbol{\pi}(\hat{\boldsymbol{M}}(t) \oplus b) \rangle}{\chi_e(b)}. \label{Eq100}
    \end{equation}
    On the other hand, we have
    \begin{eqnarray}
        \frac{\langle \hat{\boldsymbol{\mu}}(t), \boldsymbol{\pi}(\hat{\boldsymbol{M}}(t) \oplus b) \rangle}{\chi_e(b)} 
        & = & \frac{\langle \hat{\boldsymbol{\mu}}(t), \boldsymbol{\pi}(\hat{\boldsymbol{M}}(t)) + \boldsymbol{\chi}(b)\rangle}{\chi_e(b)} \label{Eq101} \\
        & = & \frac{\langle \hat{\boldsymbol{\mu}}(t) , \boldsymbol{\pi}(\hat{\boldsymbol{M}}(t)) \rangle}{{\chi_e(b)}} + \frac{\langle \hat{\boldsymbol{\mu}}(t), \boldsymbol{\chi}(b)\rangle}{{\chi_e(b)}} \nonumber \\
        & > & \frac{\langle \hat{\boldsymbol{\mu}}(t), \boldsymbol{\pi}(\hat{\boldsymbol{M}}(t))\rangle}{{\chi_e(b)}}  + \frac{\langle \boldsymbol{\mu}, \boldsymbol{\chi}(b) \rangle}{{\chi_e(b)}} - \frac{\Delta_{(d - t + 1)}}{(2 + L^2) U_\mathcal{A}} \frac{\| \boldsymbol{\chi}(b) \|}{{\chi_e(b)}} \label{Eq102} \\
        & \geq & \frac{\langle \hat{\boldsymbol{\mu}}(t), \boldsymbol{\pi}(\hat{\boldsymbol{M}}(t))\rangle}{{\chi_e(b)}}  + \frac{\langle \boldsymbol{\mu}, \boldsymbol{\chi}(b)\rangle}{{\chi_e(b)}} - \frac{\Delta_{e}}{(2 + L^2) U_\mathcal{A}} \frac{\| \boldsymbol{\chi}(b) \|}{{\chi_e(b)}} \nonumber \\
        & \geq & \frac{\langle \hat{\boldsymbol{\mu}}(t), \boldsymbol{\pi}(\hat{\boldsymbol{M}}(t))\rangle}{{\chi_e(b)}} + \frac{\langle \boldsymbol{\mu}, \boldsymbol{\chi}(b)\rangle}{{\chi_e(b)}} - \frac{\Delta_{e}}{(2 + L^2)} \label{Eq103}  \\
        & \geq & \frac{\langle \hat{\boldsymbol{\mu}}(t), \boldsymbol{\pi}(\hat{\boldsymbol{M}}(t))\rangle}{{\chi_e(b)}} + \frac{1 + L^2}{2 + L^2} \Delta_{e} \label{Eq104}  \\
        & \geq & \frac{\langle \hat{\boldsymbol{\mu}}(t), \boldsymbol{\pi}(\hat{\boldsymbol{M}}(t))\rangle}{{\chi_e(b)}}. \label{Eq105}
    \end{eqnarray}
    This means that $\langle \hat{\boldsymbol{\mu}}(t), \boldsymbol{\pi}(\hat{\boldsymbol{M}}(t) \oplus b) \rangle > \langle \hat{\boldsymbol{\mu}}(t), \boldsymbol{\pi}(\hat{\boldsymbol{M}}(t))\rangle$. 
    This contradicts the definition of $\hat{\boldsymbol{\pi}}(t)$, and therefore, we have $(e, \pi^{*}_e)\in ( \boldsymbol{M}^{*}\cap \hat{\boldsymbol{M}}(t))$.
    
\end{proof}

The next lemma takes a step further.
Hereinafter, we denote $\Tilde{\boldsymbol{M}}^e(t)$ as $\boldsymbol{M}(\Tilde{\boldsymbol{\pi}}^{e}(t))$.
\begin{lemma} \label{Lemma20}
    Fix any phase $t > 0$. Suppose that event $\tau_{t}$ occurs. Also, assume that $\boldsymbol{S}(t) \subseteq M^{*}$ and $\boldsymbol{R}(t)\cap \boldsymbol{M}^{*} = \emptyset$. Let $e\in[d] \setminus \boldsymbol{F}(t)$ satisfy $\Delta_{(d - t+ 1)} \leq \Delta_{e}$. Then, we have 
    \begin{equation}
        \frac{\langle \hat{\boldsymbol{\mu}}(t), \hat{\boldsymbol{\pi}}(t) - \Tilde{\boldsymbol{\pi}}^e(t) \rangle}{\left|\hat{\pi}_e(t) - \Tilde{\pi}^{e}_{e}(t)\right|} > \frac{L + 1/L}{2 + L^2}\Delta_{(d - t + 1)}.
    \end{equation}
\end{lemma}
\begin{proof}
    By Lemma \ref{Lemma19}, we see that 
    \begin{equation}
        (e, {\pi}^{*}_e) \in (\boldsymbol{M}^{*}\cap \boldsymbol{S}(t)).
    \end{equation}
    From the definition of $\Tilde{\boldsymbol{M}}^{e}(t)$, which ensures that $\Tilde{M}^e_e(t)\neq M^{*}_e$, we have $(e, {\pi}^{*}_e) \in (\boldsymbol{M}^{*} \setminus \Tilde{\boldsymbol{M}}^{e}(t))$. \par
    Hence, we apply Lemma \ref{InterpolationLemma}. 
    There exists 
     $b = (\boldsymbol{b}^{+}, \boldsymbol{b}^{-}) \in \mathcal{B}$ such that 
     \textbf{(a)} ${b}^{-}_e = \left(\Tilde{\boldsymbol{M}}^{e}(t) \boxminus \boldsymbol{M}^{*} \right)_e  $ or $ b^{+}_{e} = \left(\boldsymbol{M}^{*}\boxminus \Tilde{\boldsymbol{M}}^{e}(t) \right)_{e} $,
     \textbf{(b)} $\boldsymbol{b}^{-}\subseteq (\Tilde{\boldsymbol{M}}^{e}(t) \boxminus \boldsymbol{M}^{*})$, 
     \textbf{(c)} $\boldsymbol{b}^{+} \subseteq \boldsymbol{M}^{*} \boxminus \Tilde{\boldsymbol{M}}^{e}(t)$, 
    \textbf{(d)} $(\Tilde{\boldsymbol{M}}^{e}(t) \oplus b) \in \mathcal{M}$, 
    \textbf{(e)} $\boldsymbol{M}^{*} \ominus b \in \mathcal{M}$, and 
    $\frac{\langle\boldsymbol{\mu}, \boldsymbol{\chi}(b)\rangle}{\chi_e(b)} \geq \Delta_e$. \par
    Define $\overline{\boldsymbol{M}}^e(t) \triangleq \Tilde{\boldsymbol{M}}^{e}(t) \oplus b$. 
    Using Lemma \ref{Lemma18}, we have $ \boldsymbol{S}(t) \subseteq \overline{\boldsymbol{M}}^{e}(t)$ and $\boldsymbol{R}(t) \cap \overline{\boldsymbol{M}}^{e}(t) = \emptyset$. Since $\overline{\boldsymbol{M}}^{e}(t) \in \mathcal{M}$ and by definition $\hat{\boldsymbol{M}}(t) = \argmax_{\boldsymbol{M}\in \mathcal{M}, \boldsymbol{S}_t\subseteq M, \boldsymbol{R}_t\cap \boldsymbol{M} = \emptyset } \langle \hat{\boldsymbol{\mu}}, \boldsymbol{\pi}(\boldsymbol{M}) \rangle$, we have
    \begin{equation}
        \frac{\langle \hat{\boldsymbol{\mu}}(t), \boldsymbol{\pi}(\hat{\boldsymbol{M}}(t)) \rangle}{\left|\hat{\pi}_e(t) - \Tilde{\pi}^{e}_{e}(t)\right|} \geq \frac{\langle \hat{\boldsymbol{\mu}}(t), \boldsymbol{\pi} (\overline{\boldsymbol{M}}^{e}(t)) \rangle}{\left|\hat{\pi}_e(t) - \Tilde{\pi}^{e}_{e}(t)\right|}. \label{Eq107}
    \end{equation}
    Hence, we have 
    \begin{eqnarray}
        \frac{\langle \hat{\boldsymbol{\mu}}(t), \hat{\boldsymbol{\pi}}(t) - \Tilde{\boldsymbol{\pi}}^e(t) \rangle}{\left|\hat{\pi}_e(t) - \Tilde{\pi}^{e}_{e}(t)\right|} 
        & \geq & \frac{\langle \hat{\boldsymbol{\mu}}(t), \boldsymbol{\pi} (\overline{\boldsymbol{M}}^{e}(t)) - \Tilde{\boldsymbol{\pi}}^e(t) \rangle}{\left|\hat{\pi}_e(t) - \Tilde{\pi}^{e}_{e}(t)\right|} \nonumber\\
        & = & \frac{\langle \hat{\boldsymbol{\mu}}(t), \boldsymbol{\pi}(\Tilde{\boldsymbol{M}}^{e}(t) \oplus b) \rangle}{\left|\hat{\pi}_e(t) - \Tilde{\pi}^{e}_{e}(t)\right|} - \frac{\langle \hat{\boldsymbol{\mu}}(t), \boldsymbol{\pi} (\tilde{\boldsymbol{M}}^{e}(t)) \rangle}{\left|\hat{\pi}_e(t) - \Tilde{\pi}^{e}_{e}(t)\right|} \nonumber \\
        & = & \frac{\langle \hat{\boldsymbol{\mu}}(t), \boldsymbol{\pi}(\tilde{\boldsymbol{M}}^{e}(t)) + \boldsymbol{\chi}(b) \rangle}{\left|\hat{\pi}_e(t) - \Tilde{\pi}^{e}_{e}(t)\right|} - \frac{\langle \hat{\boldsymbol{\mu}}(t), \boldsymbol{\pi}(\tilde{\boldsymbol{M}}^{e}(t)) \rangle}{\left|\hat{\pi}_e(t) - \Tilde{\pi}^{e}_{e}(t)\right|} \label{Eq108} \\
        & = & \frac{\langle \hat{\boldsymbol{\mu}}(t), \boldsymbol{\chi}(b) \rangle}{\left|\hat{\pi}_e(t) - \Tilde{\pi}^{e}_{e}(t)\right|} \nonumber \\
        & > & \frac{\chi_e(b)}{\left|\hat{\pi}_e(t) - \Tilde{\pi}^{e}_{e}(t)\right|} \left( \frac{\langle \boldsymbol{\mu}, \boldsymbol{\chi}(b) \rangle}{\chi_e(b)} - \frac{\Delta_{(d - t + 1)}}{(2 + L^2)U_{\mathcal{A}}} \frac{\| \boldsymbol{\chi}(b) \|_1}{\chi_e(b)} \right) \label{Eq109} \\
        & \geq & \frac{\chi_e(b)}{\left|\hat{\pi}_e(t) - \Tilde{\pi}^{e}_{e}(t)\right|} \left( \frac{\langle \boldsymbol{\mu}, \boldsymbol{\chi}(b) \rangle}{\chi_e(b)} - \frac{\Delta_{e}}{(2 + L^2) U_{\mathcal{A}}} \frac{\| \boldsymbol{\chi}(b) \|_1}{\chi_e(b)} \right) \label{Eq110} \\
        & \geq & \frac{\chi_e(b)}{\left|\hat{\pi}_e(t) - \Tilde{\pi}^{e}_{e}(t)\right|} \cdot \frac{1 + L^2}{2 + L^2} \Delta_e \label{Eq111} \\
        & \geq &  \frac{L + 1/L}{2 + L^2} \Delta_e \nonumber \\
        & \geq & \frac{L + 1/L}{2 + L^2} \Delta_{(d - t + 1)}, \label{Eq112}
    \end{eqnarray}
    where (\ref{Eq108}) follows from Lemma \ref{Lemma1}; (\ref{Eq109}) follows from Lemma \ref{Lemma17}, the assumption on event $\tau_{t}$; (\ref{Eq110}) follows from the assumption that $\Delta_{e} \geq \Delta_{(d - t + 1)}$; (\ref{Eq111}) holds since $b \in \mathcal{B}$ and therefore $\frac{\| \boldsymbol{\chi}(b) \|_1}{\left|\hat{\pi}_e(t) - \Tilde{\pi}^{e}_{e}(t)\right|} \leq U_\mathcal{A}$; (\ref{Eq112}) follows from the fact that $\frac{\langle \boldsymbol{\mu}, \boldsymbol{\chi}(b) \rangle}{\left|\hat{\pi}_e(t) - \Tilde{\pi}^{e}_{e}(t)\right|} \geq \Delta_e$.
\end{proof}

The next lemma shows that, during phase $t$, if $\Delta_{e} \leq \Delta_{(d - t + 1)}$ for some $e$, then the empirical gap between $\boldsymbol{\pi}(t)$ and $\Tilde{\boldsymbol{\pi}}^{e}(t)$ is smaller than $\frac{1}{3} \Delta_{(d - t + 1)}$.
\begin{lemma} \label{Lemma21}
    Fix any phase $t > 0$. Suppose that event $\tau_{t}$ occurs. Also, assume that $\boldsymbol{S}(t) \subseteq \boldsymbol{M}^{*}$ and $\boldsymbol{R}(t) \subseteq \boldsymbol{M}^{*} = \emptyset$. Suppose an active arm $e \in [d] \setminus \boldsymbol{F}(t)$ satisfies that $M^{*}_e \neq M_e(t)$. Then, we have
    \begin{equation}
        \frac{\langle \hat{\boldsymbol{\mu}}(t), \hat{\boldsymbol{\pi}}(t) - \Tilde{\boldsymbol{\pi}}^{e}(t) \rangle}{\left|\hat{\pi}_e(t) - \Tilde{\pi}^{e}_{e}(t)\right|} \leq \frac{L}{2 + L^2} \Delta_{(d - t + 1)},
    \end{equation}
    where $b = $.
\end{lemma}
\begin{proof}
    Fix any exchange class $\mathcal{B} = \argmin _{\mathcal{B}' \in \mathrm{Exchange(\mathcal{M})}} \mathrm{width}(\mathcal{B}')$. \par
    From the assumption that $M^{*}_e \neq \hat{M}_{e}(t)$, we can apply Lemma \ref{InterpolationLemma}, and have 
    \textbf{(a)}$ b^{-}_e = (\hat{\boldsymbol{M}}(t) \boxminus \boldsymbol{M}^{*})_{e} $ or $ {b}^{+}_e = (\boldsymbol{M}^{*} \boxminus \hat{\boldsymbol{M}}(t))_e$, 
    \textbf{(b)} $ \boldsymbol{b}^{-} \subseteq \hat{\boldsymbol{M}}(t) \boxminus \boldsymbol{M}^{*}$, 
    \textbf{(c)} $ \boldsymbol{b}^{+} \subset \boldsymbol{M}^{*} \boxminus \hat{\boldsymbol{M}}(t)$, 
    \textbf{(d)} $\hat{\boldsymbol{M}}(t)\oplus b \in \mathcal{M}$ 
    \textbf{(e)} $\boldsymbol{M}^{*} \ominus b \in \mathcal{M}$, and 
    $\frac{\langle \boldsymbol{\mu}, \boldsymbol{\chi}(b) \rangle}{\chi_e(b)} \geq \Delta_{e} > 0$. \par
    Define $\overline{\boldsymbol{M}}^{e}(t) \triangleq \hat{\boldsymbol{M}}(t) \oplus b$, and let $\overline{\boldsymbol{\pi}}^{e}(t) = \boldsymbol{\pi}(\overline{\boldsymbol{M}}^{e}(t))$.
    We claim that 
    \begin{equation}
        \langle \hat{\boldsymbol{\mu}}(t), \Tilde{\boldsymbol{\pi}}^{e}(t) \rangle \geq \langle \hat{\boldsymbol{\mu}}, \overline{\boldsymbol{\pi}}^{e}(t) \rangle. \label{Eq113}
    \end{equation}
    From the definition of $\Tilde{\boldsymbol{M}}^{e}(t)$ in Algorithm \ref{CSAAlgorithm}, we only need to show that (\textbf{a}): $\hat{\pi}_e(t) \neq \overline{\pi}_{e}(t)$  and (\textbf{b}): $\boldsymbol{S}(t) \subseteq \overline{\boldsymbol{M}}^{e}(t)$ and $\boldsymbol{R}(t) \cap \overline{\boldsymbol{M}}^{e}(t) = \emptyset$. Since, either $\boldsymbol{b}^{+}$ or $\boldsymbol{b}^{^-}$ has an $e$-changer, the $e$-th element of $\overline{\boldsymbol{\pi}}(t)$ is different from that of $\hat{\boldsymbol{\pi}}(t)$.
    Next, we notice that this follows directly from Lemma \ref{Lemma18} by setting $M = \hat{\boldsymbol{M}}(t)$. Hence, we have shown that (\ref{Eq113}) holds. \par
    Therefore, we have
    \begin{eqnarray}
        \frac{\langle \hat{\boldsymbol{\mu}}(t), \hat{\boldsymbol{\pi}}(t) - \Tilde{\boldsymbol{\pi}}^{e}(t) \rangle}{\left|\hat{\pi}_e(t) - \Tilde{\pi}^{e}_{e}(t)\right|}
        & \leq & \frac{\langle \hat{\boldsymbol{\mu}}(t), \hat{\boldsymbol{\pi}}(t) - \overline{\boldsymbol{\pi}}^{e}(t) \rangle}{\left|\hat{\pi}_e(t) - \Tilde{\pi}^{e}_{e}(t)\right|} \nonumber \\
        & \leq & \frac{\langle \hat{\boldsymbol{\mu}}(t), \boldsymbol{\pi}(t) -  \left(\boldsymbol{\pi}(t) + \boldsymbol{\chi}(b) \right) \rangle}{\left|\hat{\pi}_e(t) - \Tilde{\pi}^{e}_{e}(t)\right|} \nonumber \\
        & = & - \frac{\langle \hat{\boldsymbol{\mu}}(t), \boldsymbol{\chi}(b) \rangle}{\left|\hat{\pi}_e(t) - \Tilde{\pi}^{e}_{e}(t)\right|} \nonumber \\
        & \leq & \frac{\chi_e(b)}{\left|\hat{\pi}_e(t) - \Tilde{\pi}^{e}_{e}(t)\right|} \cdot \left( - \frac{\langle {\boldsymbol{\mu}}, \boldsymbol{\chi}(b) \rangle}{\chi_e(b)}  + \frac{\Delta_{(d - t + 1)}}{(2 + L^2) \mathrm{U_\mathcal{M}}} \frac{\| \boldsymbol{\chi}(b) \|_1}{\chi_e(b)} \right)  \label{Eq115} \\
        & \leq & \frac{\chi_e(b)}{\left|\hat{\pi}_e(t) - \Tilde{\pi}^{e}_{e}(t)\right|} \cdot \frac{\Delta_{(d - t + 1)}}{(2 + L^2)} \label{Eq116} \\
        & \leq & \frac{L}{2 + L^2} \Delta_{(d - t + 1)} \nonumber
    \end{eqnarray}
\end{proof}

\subsection{Proof of Theorem \ref{CSATheorem}}
For the reader's convenience, we first restate Theorem \ref{CSATheorem} as follows.
\CSATheorem*
\begin{proof}
    First, we show that the algorithm takes at most $T$ samples. Note that exactly one arm is pulled for $\Tilde{T}_1$ times, one arm is pulled $\tilde{T}_2$ times, ..., and one arm is pulled $\tilde{T}_{d}$ times. Therefore, the total number of samples used by the algorithm is bounded by 
    \begin{eqnarray}
        \sum_{t = 1}^{d} \tilde{T}_{t} 
        &\leq& \sum_{t = 1}^{d} \left( \frac{T - d}{\tilde{\log} (d) (d - t + 1)}  + 1\right) \nonumber \\
        & = & \frac{T - d}{\tilde{\log}(d)} \tilde{\log} (d) + d = T. \nonumber 
    \end{eqnarray}
    By Lemma \ref{Lemma16}, we know that the event $\tau = \bigcap_{t = 1}^{T} \tau_{t}$ occurs with probability at least $ 1- d^{2} \exp\left( \frac{T - d}{R^2 \tilde{\log}(d) U^2_{\mathcal{A}} \mathbf{H}_2 } \right)$. Therefore, we only need to prove that, under event $\tau$ the algorithm outputs $\boldsymbol{M}^{*}$. We will assume that event $\tau$ occurs in the rest of the proof. \par
    We prove the theorem by induction. Fix a phase $t\in[d]$. Suppose that the algorithm does not make any error before phase $t$, i.e., $\boldsymbol{S}(t) \subseteq \boldsymbol{M}^{*}$ and $\boldsymbol{R}(t) \cap \boldsymbol{M}^{*} \neq \emptyset$. We show that the algorithm does not err at phase $t$. \par
    At the beginning of phase $t$, there are exactly $t - 1$ inactive arms, i.e., $|F(t)| = t - 1$. Therefore, there must exist an active arm $e(t) \in [d] \setminus \boldsymbol{F}(t)$, such that $\Delta_{e(t)} = \Delta_{(d - t + 1)}$. 
    Hence, by Lemma \ref{Lemma20}, we have 
    \begin{eqnarray}
        \frac{\langle \hat{\boldsymbol{\mu}}(t), \boldsymbol{\pi}(t) - \Tilde{\boldsymbol{\pi}}^{e(t)}(t) \rangle}{\left| \hat{\pi}_{e(t)}(t) - \tilde{\pi}^{e(t)}_{e(t)}(t)  \right|} \geq \frac{L + 1/L}{2 + L^2} \Delta_{(d - t + 1)}, \label{Eq117}
    \end{eqnarray}
    where $b$ is an exchange set of $\hat{\boldsymbol{M}}(t)$ and $\Tilde{\boldsymbol{M}}^{e}(t)$. \par
    Notice that the algorithm makes an error in phase $t$ if and only if $(p(t), \hat{\pi}_{p(t)}(t) ) \in (\boldsymbol{M}^{*} \cap \lnot \hat{\boldsymbol{M}}(t)) \cup (\lnot \boldsymbol{M}^{*} \cap \hat{\boldsymbol{M}}(t))$. \par
    Suppose that $(p(t), \hat{\pi}_{p(t)}(t) ) \in (\boldsymbol{M}^{*} \cap \lnot \hat{\boldsymbol{M}}(t)) \cup (\lnot \boldsymbol{M}^{*} \cap \hat{\boldsymbol{M}}(t))$. 
    From Lemma \ref{Lemma21}, we have
    \begin{eqnarray}
        \frac{\langle \hat{\boldsymbol{\mu}}(t), \hat{\boldsymbol{\pi}}(t) - \Tilde{\boldsymbol{\pi}}^{p(t)}(t) \rangle}{\left| \hat{\pi}_{p(t)} - \tilde{\pi}^{p(t)}_{p(t)} \right|} \leq \frac{L}{2 + L^2} \Delta_{(d - t + 1)}. \label{Eq118}
    \end{eqnarray}
    By combining (\ref{Eq117}) and (\ref{Eq118}), we see that 
    \begin{eqnarray}
        \frac{\langle \hat{\boldsymbol{\mu}}(t), \hat{\boldsymbol{\pi}}(t) - \Tilde{\boldsymbol{\pi}}^{p(t)}(t) \rangle}{\left| \hat{\pi}_{p(t)}(t) - \tilde{\pi}^{p(t)}_{p(t)}(t) \right|} \leq \frac{L}{2 + L^2} \Delta_{(d - t + 1)} < \frac{L+1/L}{2 + L^2} \Delta_{(d - t + 1)} \leq \frac{\langle \hat{\boldsymbol{\mu}}(t), \hat{\boldsymbol{\pi}}(t) - \Tilde{\boldsymbol{\pi}}^{e(t)}(t) \rangle}{\hat{\pi}_{e(t)}(t) - \Tilde{\pi}^{e(t)
        }_{e(t)}(t)}. \label{Eq119}
    \end{eqnarray}
    However, (\ref{Eq119}) is contradictory to the definition of 
    \begin{eqnarray}
        p(t) = \argmax_{e \in [d] \setminus F(t)} \frac{\langle \hat{\boldsymbol{\mu}}(t), \hat{\boldsymbol{\pi}}(t) - \Tilde{\boldsymbol{\pi}}(t) \rangle}{\hat{\pi}_{e}(t) - \Tilde{\pi}^{e}_{e}(t)}.
    \end{eqnarray}
    Therefore, we have proven that $(p(t), \hat{\pi}_{p(t)}(t) ) \notin (\boldsymbol{M}^{*} \cap \lnot \hat{\boldsymbol{M}}(t)) \cup (\lnot \boldsymbol{M}^{*} \cap \hat{\boldsymbol{M}}(t))$. This means that the algorithm does not err at phase $t$, or equivalently $\boldsymbol{S}(t + 1) \subseteq \boldsymbol{M}^{*}$ and $\boldsymbol{R}(t + 1) \cap \boldsymbol{M}^{*} = \emptyset$. By induction, we have proven that the algorithm does not err at any phase $t \in [d]$. \par
    Hence, we have $\boldsymbol{S}(d + 1) \subseteq \boldsymbol{M}^{*}$ and $\boldsymbol{R}(d + 1) \subseteq \lnot \boldsymbol{M}^{*}$ in the final phase. This means that $\boldsymbol{S}(d + 1) = \boldsymbol{M}^{*}$, and therefore, $\boldsymbol{\pi}^{\mathrm{out}} = \boldsymbol{\pi}^{*}$ after phase $d$.
\end{proof}

\section{Proof of Theorem \ref{Minimax-CombSAR_Algorithm_MainTheorem}}
We first introduce some useful lemmas to prove Theorem \ref{Minimax-CombSAR_Algorithm_MainTheorem}. Lemma \ref{Lemma1_for_our_algorithm} shows that Algorithm \ref{Minimax-CombSAR_Algorithm} pulls arms no more than $T$ times. Recall that $B = 2^{ \left\lceil \log_2 d \right \rceil} - 1$ and $T' = T - \left\lfloor \frac{T}{d}\beta \right \rfloor \times d$.
\begin{lemma}\label{Lemma1_for_our_algorithm}
    Algorithm \ref{Minimax-CombSAR_Algorithm} terminates in phase $\lceil \log_2 d\rceil$ with no more than a total of $T$ arm pulls.
\end{lemma}
\begin{proof}
    The total number of arm pulls $T_{\mathrm{total}}$ is bounded as follows.
    \begin{eqnarray}
        T_{\mathrm{total}} & = & \left\lfloor \frac{T}{d}\beta \right \rfloor \times d + \sum_{r = 1}^{\lceil \log_2 d \rceil} \sum_{s \in [d]} \lceil p_s(r)\cdot m(r) \rceil \nonumber \\
          & \leq & \left\lfloor \frac{T}{d}\beta \right \rfloor \times d + \sum_{r = 1}^{\lceil \log_2 d \rceil} (d + \frac{T'- d\lceil \log_2 d \rceil}{B/2^{r - 1}}) \nonumber \\
          & \leq &\left\lfloor \frac{T}{d}\beta \right \rfloor \times d +  d\lceil \log_2 d \rceil + \frac{2^{\lceil \log_2 d \rceil} - 1}{B}(T'- d\lceil \log_2 d \rceil) \nonumber \\ 
          & = &  \nonumber T.
    \end{eqnarray}
\end{proof}
Let us write $\Delta'_{(i)} = \langle \boldsymbol{\mu}, \boldsymbol{\pi}^{1} - \boldsymbol{\pi}^{i} \rangle$. Note that $\boldsymbol{\mu}^{\top} \boldsymbol{\pi}^{i} \geq \boldsymbol{\mu}^{\top} \boldsymbol{\pi}^{i + 1} $ for all $i \in [K - 1]$. Lemma \ref{Lemma2_for_our_algorithm} bounds the probability that a certain action has its estimate of the expected reward larger than that of the best action at the end of phase $r$.
\begin{lemma} \label{Lemma2_for_our_algorithm}
    For a fixed realization of $\mathbb{A}(r)$ satisfying $\boldsymbol{\pi}^{*} \in \mathbb{A}(r)$, for any action $\boldsymbol{\pi}^{i} \in \mathbb{A}(r)$,
    \begin{equation}
        \Pr\left[ \hat{\boldsymbol{\mu}}^{\top}(r + 1)\boldsymbol{\pi}^{1} < \hat{\boldsymbol{\mu}}^{\top}(r + 1) \boldsymbol{\pi}^{i}  \right] \leq \exp\left( - \frac{2{\Delta_{(i)}'}^{2}}{ \sum\limits_{s = 1}^{d}  \frac{(\pi^{1}_s - {\pi}_s^{i} )^2}{ T_{s}(r + 1) } R^2} \right).
    \end{equation}
\end{lemma}
\begin{proof}
For any $i \in \{2, \ldots, K\}$, we have 
\begin{eqnarray}
    \Pr\left[ \hat{\boldsymbol{\mu}}^{\top}(r + 1)\boldsymbol{\pi}^{1}< \hat{\boldsymbol{\mu}}^{\top}(r + 1)\boldsymbol{\pi}^{i}  \right] 
    & \leq & \Pr\left[\left\langle \hat{\boldsymbol{\mu}}(r + 1) -\boldsymbol{\mu} ,  \boldsymbol{\pi}^{1} - \boldsymbol{\pi}^{i} \right\rangle < -\Delta'_{(i)} \right] \nonumber \\
    & \leq & \exp\left( - \frac{2{\Delta_{(i)}'}^{2}}{ \sum\limits_{s = 1}^{d}  \frac{(\pi^{1}_s - {\pi}_s^{i} )^2}{ T_{s}(r + 1) } R^2} \right), \label{Lemma10Inequality}
\end{eqnarray}
where the last inequality follows from Hoeffding's inequality \citep{Hoeffding1963}. \par
If we use allocation vector (\ref{min-max-p(r)_Lagrange}), (\ref{Lemma10Inequality}) can be upper bounded by
\begin{eqnarray}
    (\ref{Lemma10Inequality}) 
    & \leq & \exp\left( - \frac{2{\Delta'_{(i)}}^{2}}{ \sum\limits_{e = 1}^{d} \frac{(\pi^{1}_e - {\pi}_e^{i} )^2}{  \lceil p_s(r) m(r) \rceil  } R^2} \right) \\
    &\leq & \exp\left( -\frac{2{\Delta'_{(i)}}^2}{R^2{V'}^2}\cdot \frac{T' - \lceil\log_2 d \rceil }{B/2^{r - 1}} \right)  \nonumber \\
    &\leq & \exp\left( -\frac{{\Delta'_{(i)}}^2}{R^2{V'}^2}\cdot \frac{T' - \lceil\log_2 d \rceil }{d/2^{r - 1}} \right)  \nonumber
\end{eqnarray}
where $V' = \sum\limits_{e = 1}^{d} \left| \pi_{e}^{1} - \pi_{e}^{i} \right|$.
\end{proof}
Next, we bound the error probability of a single phase $r$ in Lemma \ref{Lemma3_for_our_algorithm}.
\begin{lemma}\label{Lemma3_for_our_algorithm}
   Assume that the best action is not eliminated prior to phase $r$, i.e., $\boldsymbol{\pi}^{1} \notin \mathbb{A}(r)$. Then, the probability that the best action is eliminated in phase $r + 1$ is bounded as 
   \begin{eqnarray}
        \Pr\left[ \boldsymbol{\pi}^{1} \notin \mathcal{A}_{r} | \boldsymbol{\pi}^{1} \in \mathcal{A}_{r - 1} \right] \leq \left\{
        \begin{array}{ll}
        \frac{4K}{d} \exp \left( \frac{T' - \left\lceil \log_2 d \right\rceil}{R^2 V^2} \cdot \frac{{\Delta_{(i_r)}}^2}{i_r} \right) & (\text{when} \ r = 1)\\
        3 \exp \left( \frac{T' - \left\lceil \log_2 d \right\rceil}{R^2 V^2} \cdot \frac{{\Delta_{(i_r)}}^2}{i_r} \right)  & (\text{when} \ r > 1),
        \end{array}
        \right.
    \end{eqnarray}
   where $i_r = \left\lceil \frac{d}{2^{r + 1}} \right\rceil + 1$ and 
   \begin{eqnarray}
       V = \max_{\boldsymbol{\pi} \in \mathcal{A}\setminus \{ \boldsymbol{\pi}^{*} \}, s \in \{ e \in [d] \ | \ \pi^{*}_e \neq \pi_e \ \} } \frac{\sum_{u = 1}^{d} \left| \pi^{*}_{u} - \pi_{u} \right|}{\left| \pi^{*}_{s} - \pi_{s} \right|}.
   \end{eqnarray}
\end{lemma}
\begin{proof}
    Define $\mathbb{B}(r + 1)$ as the set of actions in $\mathbb{A}(r)$ excluding the best action and $\lceil \frac{d}{2^{r + 1}} \rceil - 1$ suboptimal actions with the largest expected rewards. We have $|\mathbb{B}(r + 1)| = |\mathbb{A}(r)| - \lceil \frac{d}{2^{r + 1}} \rceil$. \par
    Let us think of bijective function $p: [d] \rightarrow [d]$, which satisfies the following:
    \begin{eqnarray}
        \forall s\in[d], \ \Delta_{(p(s))} = \Delta_{s}.
    \end{eqnarray}
    Also, we denote the inverse mapping of $f$ by $f^{-1}$. Next, for any $i \in [d]$, we define $s(i)$ as follows:
    \begin{eqnarray}
        s(i) = \argmax_{s \in [d]} | \pi^{1}_{s} - \pi^{i}_{s} |. \nonumber
    \end{eqnarray}
    Then, from the definition of $\{ \Delta_{s} \}_{s = 1, \ldots, d}$,  we have
    \begin{eqnarray}
        \frac{\Delta'_{(i)}}{| \pi^{1}_{s(i)} - \pi^{i}_{s(i)} |} \geq \Delta_{s(i)} = \Delta_{(p(s(i)))}. \label{Delta_Change}
    \end{eqnarray}
    Also, we have
    \begin{eqnarray}
        \min_{i \in \mathbb{B}(r + 1)} \Delta_{(p(s(i)))} \geq \Delta_{\left(\left\lceil \frac{d}{2^{r + 1} } \right\rceil +1\right)}. \label{Delta_Bound}
    \end{eqnarray}
    If the best action is eliminated in phase $r$, then at least $\lceil \frac{d}{2^{r}} \rceil - \lceil \frac{d}{2^{r + 1}}\rceil +1 $ actions of $\mathbb{B}(r + 1)$ have their estimates of the expected rewards larger than that of the best action. \par
    Let $N(r)$ denote the number of actions in $\mathbb{B}(r + 1)$ whose estimates of the expected rewards are larger than that of the best action. By Lemma \ref{Lemma2_for_our_algorithm}, we have
    \begin{eqnarray}
        \mathbb{E}\left[ N_r \right] 
        & = &  \sum_{i\in\mathcal{B}_r}\Pr\left[ \hat{\boldsymbol{\mu}}^{\top}(r + 1)\boldsymbol{\pi}^{1}< \hat{\boldsymbol{\mu}}^{\top}(r + 1) \boldsymbol{\pi}^{i}  \right] \nonumber \\
        & \leq & \sum_{i\in\mathcal{B}_r} \exp\left( -\frac{{\Delta'_{(i)}}^2}{R^2{V'}^2}\cdot \frac{T' - \lceil\log_2 d \rceil }{d/2^{r - 1}} \right) \nonumber \\
        & \leq & |\mathcal{B}_r| \max_{i \in \mathbb{B}(r + 1)} \exp\left( -\frac{{\Delta'_{(i)}}^2}{R^2{V'}^2}\cdot \frac{T' - \lceil\log_2 d \rceil }{d/2^{r - 1}} \right)  \nonumber \\
        & = & |\mathcal{B}_r| \max_{i \in \mathbb{B}(r + 1)} \exp\left( -\frac{| \pi^{1}_{s(i)} - \pi^{i}_{s(i)} |^{2} \Delta^{2}_{(p(s(i))}}{R^2{V'}^2}\cdot \frac{T' - \lceil\log_2 d \rceil }{d/2^{r - 1}} \right) \label{EqMinimax0} \\
        & \leq & \left( \left| \mathcal{A}_{r - 1} \right| - \left\lceil \frac{d}{2^{r + 1}}  \right\rceil \right) \exp \left( \frac{T' - \left\lceil \log_2 d \right\rceil}{R^2 V^2} \cdot \frac{{\Delta_{\left(\left\lceil \frac{d}{2^{r + 1} )} \right\rceil +1 \right)}}^{2}}{\left\lceil \frac{d}{2^{r + 1} )} \right\rceil +1 } \right) \label{EqMinimax1} \\
        & \leq & \left( \left| \mathcal{A}_{r - 1} \right| - \left\lceil \frac{d}{2^{r + 1}}  \right\rceil \right) \exp \left( \frac{T' - \left\lceil \log_2 d \right\rceil}{R^2 V^2} \cdot \frac{{\Delta_{\left(\left\lceil \frac{d}{2^{r + 1} )} \right\rceil +1 \right)}}^{2}}{\left\lceil \frac{d}{2^{r + 1} )} \right\rceil +1 } \right), \nonumber
    \end{eqnarray}
    where (\ref{EqMinimax0}) follows from (\ref{Delta_Change}), (\ref{EqMinimax1}) follows from (\ref{Delta_Bound}), and 
    \begin{eqnarray}
        V = \max_{\boldsymbol{\pi} \in \mathcal{A}\setminus \{ \boldsymbol{\pi}^{*} \}, s \in \{ e \in [d] \ | \ \pi^{*}_e \neq \pi_e \ \} } \frac{\sum_{u = 1}^{d} \left| \pi^{*}_{u} - \pi_{u} \right|}{\left| \pi^{*}_{s} - \pi_{s} \right|}.
    \end{eqnarray}
    Then, together with Markov's inequality, we obtain
    \begin{eqnarray}
        \Pr\left[ \boldsymbol{\pi}^{1} \notin \mathcal{A}_r \right] 
        & \leq & \Pr\left[ N_r \geq \left\lceil \frac{d}{2^{r}} \right\rceil -  \left\lceil \frac{d}{2^{r + 1}} \right\rceil + 1 \right]  \nonumber \\
        & \leq & \frac{\mathbb{E} \left[ N_{r} \right] }{ \left\lceil \frac{d}{2^{r}} \right\rceil -  \left\lceil \frac{d}{2^{r + 1}} \right\rceil + 1 } \nonumber \\
        & \leq & \frac{ \left| \mathcal{A}_{r - 1} \right| - \left\lceil \frac{d}{2^{r + 1}}  \right\rceil }{ \left\lceil \frac{d}{2^{r}} \right\rceil -  \left\lceil \frac{d}{2^{r + 1}} \right\rceil + 1 } \exp \left( \frac{T' - \left\lceil \log_2 d \right\rceil}{R^2 V^2} \cdot \frac{\Delta_{\left(\left\lceil \frac{d}{2^{r + 1} )} \right\rceil +1 \right)}}{\left\lceil \frac{d}{2^{r + 1} )} \right\rceil +1 } \right)
    \end{eqnarray}
    When $r = 1$, we have $\left| \mathcal{A}_{r - 1}\right| = K$. Thus,
    \begin{eqnarray}
        \frac{ \left| \mathcal{A}_{r - 1} \right| - \left\lceil \frac{d}{2^{r + 1}}  \right\rceil }{ \left\lceil \frac{d}{2^{r}} \right\rceil -  \left\lceil \frac{d}{2^{r + 1}} \right\rceil + 1 } 
        & = & \frac{K - \left\lceil \frac{d}{2^{r + 1}}  \right\rceil}{ \left\lceil \frac{d}{2^{r}} \right\rceil -  \left\lceil \frac{d}{2^{r + 1}} \right\rceil + 1 } \nonumber \\
        & \leq & \frac{K}{\frac{d}{2} - \frac{d}{2^2}} \nonumber \\
        & \leq & \frac{4K}{d}.
    \end{eqnarray}
    When $r > 1$, we have $\left| \mathcal{A}_{r - 1} \right| = \left\lceil \frac{d}{2^{r - 1}} \right\rceil $. Thus,
    \begin{eqnarray}
        \frac{ \left| \mathcal{A}_{r - 1} \right| - \left\lceil \frac{d}{2^{r + 1}}  \right\rceil }{ \left\lceil \frac{d}{2^{r}} \right\rceil -  \left\lceil \frac{d}{2^{r + 1}} \right\rceil + 1 } 
        & = & \frac{ \left\lceil \frac{d}{2^{r - 1}}  \right\rceil - \left\lceil \frac{d}{2^{r + 1}}  \right\rceil }{ \left\lceil \frac{d}{2^{r}} \right\rceil -  \left\lceil \frac{d}{2^{r + 1}} \right\rceil + 1 } \nonumber \\
        & \leq & \frac{ \frac{d}{2^{r - 1}} + 1  - \left\lceil \frac{d}{2^{r + 1}}  \right\rceil }{ \frac{d}{2^{r}} -  \left\lceil \frac{d}{2^{r + 1}} \right\rceil + 1 } \nonumber \\
        & \leq & \frac{ 3 \cdot \frac{d}{2^{r + 1}} + \frac{d}{2^{r + 1}} -  \left\lceil \frac{d}{2^{r + 1}} \right\rceil + 1 }{ \frac{d}{2^{r + 1}} + \frac{d}{2^{r + 1}} -  \left\lceil \frac{d}{2^{r + 1}} \right\rceil + 1 } \nonumber \\ 
        & \leq & 3, 
    \end{eqnarray}
    where the last inequality results from the fact that for any $x, y > 0$, $\frac{3x + y}{x + y} \leq 3$. \par
    Therefore, for this specific realization of $\mathcal{A}_{r - 1}$ satisfying $1 \in \mathcal{A}_{r - 1}$,
    \begin{eqnarray}
        \Pr\left[ \boldsymbol{\pi}^{1} \notin \mathcal{A}_{r} | \boldsymbol{\pi}^{1} \in \mathcal{A}_{r - 1} \right] \leq \left\{
        \begin{array}{ll}
        \frac{4K}{d} \exp \left( \frac{T' - \left\lceil \log_2 d \right\rceil}{R^2 V^2} \cdot \frac{{\Delta_{(i_r)}}^2}{i_r} \right) & (\text{when} \ r = 1)\\
        3 \exp \left( \frac{T' - \left\lceil \log_2 d \right\rceil}{R^2 V^2} \cdot \frac{{\Delta_{(i_r)}}^2}{i_r} \right)  & (\text{when} \ r > 1),
        \end{array}
        \right.
    \end{eqnarray}
    where $i_r = \lceil \frac{d}{2^{r + 1}} +1 \rceil$.
\end{proof}
Finally, we prove Theorem \ref{Minimax-CombSAR_Algorithm_MainTheorem}.
\begin{proof}[Proof of Theorem \ref{Minimax-CombSAR_Algorithm_MainTheorem}]
    We have
    \begin{eqnarray}
        \Pr\left[ \boldsymbol{\pi}^{\mathrm{out}} \neq \boldsymbol{\pi}^{1} \right] 
        & = & \Pr\left[ \boldsymbol{\pi}^{1} \notin \mathcal{A}_{\lceil \log_2 d \rceil} \right] \nonumber \\
        & \leq & \sum_{r = 1}^{\lceil \log_2 d\rceil}  \Pr\left[ \boldsymbol{\pi}^{1} \notin \mathcal{A}_{r} | \boldsymbol{\pi}^{1} \in \mathcal{A}_{r - 1} \right] \nonumber \\
        & \leq & \frac{4K}{d} \exp \left( \frac{T' - \left\lceil \log_2 d \right\rceil}{R^2 V^2} \cdot \frac{{\Delta}^2_{(i_{1})}}{i_{1}} \right) + \sum_{s = 1}^{d} 3 \exp \left( \frac{T' - \left\lceil \log_2 d \right\rceil}{R^2 V^2} \cdot \frac{{\Delta}^2_{(i_{r})}}{i_{r}} \right) \nonumber \\
        & \leq & \left( \frac{4K}{d} + 3\left( \lceil \log_2 d \rceil - 1 \right) \right) \exp \left( \frac{T' - \left\lceil \log_2 d \right\rceil}{R^2 V^2} \cdot \frac{1}{\max_{1 \leq s \leq d} \frac{s}{{\Delta_{(s)}}^2}} \right) \nonumber \\
        & < & \left( \frac{4K}{d} + 3 \log_2 d \right) \exp \left( \frac{T' - \left\lceil \log_2 d \right\rceil}{R^2 V^2} \cdot \frac{1}{\mathbf{H}_{2}} \right),
    \end{eqnarray}
    where 
    \begin{eqnarray}
        \mathbf{H}_{2} = \max_{1 \leq s \leq d} \frac{s}{\Delta_{(s)}^2}. \nonumber
    \end{eqnarray}
\end{proof}

\end{document}